\documentclass[11pt,letterpaper]{article} 

 \usepackage[margin=1in]{geometry}
\usepackage[utf8]{inputenc} 
\usepackage[T1]{fontenc}    
\usepackage{dsfont}
\usepackage[maxfloats=100]{morefloats}[2015/07/22]
\usepackage{multirow}
\usepackage{wrapfig}
\usepackage[T1]{fontenc}
\usepackage{lmodern}
\usepackage{float}
\usepackage{booktabs}       
\usepackage{amsfonts}       
\usepackage{nicefrac}       
\usepackage{microtype}      
\usepackage{enumerate}
\usepackage{lipsum}
\usepackage{mathtools}
\usepackage{cuted}
\usepackage{float}
\usepackage[dvipsnames,table,xcdraw]{xcolor}
\usepackage{bbm}
\usepackage{varioref}
\usepackage{hyperref}
                                
\usepackage{amssymb} 
\usepackage{rotating}
\usepackage{graphicx}
\usepackage{subcaption}
\usepackage[normalem]{ulem}
\usepackage{placeins}

\usepackage{amsthm}                                 
\usepackage{url}
\usepackage{amsmath}
\usepackage{amssymb}
\usepackage{mathtools}
\usepackage{hyperref}
\newtheorem{theorem}{Theorem}[section]
\newtheorem{proposition}[theorem]{Proposition}
\newtheorem{lemma}[theorem]{Lemma}

\newtheorem{assumption}[theorem]{Assumption}

\newtheorem{remark}[theorem]{Remark}
\usepackage{algorithm}
\usepackage{algpseudocode}
\usepackage{xcolor}
\usepackage{graphicx}
\usepackage{subcaption}
\newcommand{\me}[1]{{\color{black}#1}}
\newcommand{\fy}[1]{{\color{black}#1}}
\newcommand{\mee}[1]{{\color{black}#1}}

\newcommand{\Rme}[1]{{\color{black}#1}}
\newcommand{\mje}[1]{{\color{black}#1}}
\newcommand{\mjee}[1]{{\color{black}#1}}

\title{\LARGE \bf
A Locally Penalized Cross-Estimate Federated Method with Guarantees for Constrained Personalized Learning
}

\author{Mohammadjavad Ebrahimi$^{1}$, Daniel Burbano$^{2}$, and Farzad Yousefian$^{1}$
\thanks{This work was funded in part by the ONR under grant N00014-22-1-2757, and in part by the DOE under grants DE-SC0023303 and DE-SC0025570 and in part by the NSF (CMMI 2443301).}
\thanks{$^{1}$Ebrahimi and Yousefian are with the Department of Industrial and Systems Engineering, Rutgers University, USA. \{mohammadjavad.ebrahimi,farzad.yousefian\}@rutgers.edu. $^{2}$Burbano is with the
Department of Electrical and Computer Engineering, Rutgers University, USA. daniel.burbano@rutgers.edu.}
}

\begin{document}
\sloppy

\maketitle
\thispagestyle{empty}
\pagestyle{empty}

\begin{abstract}
\mjee{Federated learning (FL) is a communication-efficient framework for solving distributed optimization problems.} \mje{Standard FL methods aggregate local updates into a single global model that is shared by all agents, effectively imposing consensus.} Under heterogeneous local constraints, however, a common feasible model may be overly restrictive or may not exist. \mje{Existing constrained FL methods largely retain this shared-model structure and therefore do not directly address personalization under heterogeneous agent-specific feasible sets.} We study a \mjee{constrained personalized} FL problem that assigns a distinct feasible model to each agent while coupling the models through a collaborative objective. We propose \mjee{Locally Penalized Cross-Estimate Federated Averaging (PCE-FedAvg), where} each agent maintains a multi-block vector containing its own model and estimates of the other agents' models while applying the feasibility penalty \mjee{only locally} to its own block. The server aggregates corresponding blocks separately, thereby preserving personalization \mje{without requiring agents to explicitly disclose their local constraint sets}. \mjee{For any $\epsilon>0$,} we establish finite-time upper and lower suboptimality bounds and agent-wise squared infeasibility bounds, \mje{yielding communication complexities of $\mathcal{O}(\epsilon^{-2})$ and $\mathcal{O}(\epsilon^{-1})$, respectively}. Experiments on MNIST and CIFAR-10 support our theoretical results.
\end{abstract}

\section{INTRODUCTION}

\noindent Distributed optimization enables networked agents to cooperatively solve learning and decision problems using locally available information, with applications in robotics, sensing, energy systems, and control~\cite{majcherczyk2021flow,predd2006distributed,lin2022ppfl_tsg,berkenkamp2017safe}. Federated learning (FL) \mjee{has emerged as a communication-efficient framework in distributed optimization} in which agents collaborate without centralizing their private data~\cite{mcmahan2017communication,kairouz2021advances}. Classical FL minimizes $\tfrac{1}{m}\textstyle\sum_{i=1}^m f_i(x)$ over a single shared model~\cite{mcmahan2017communication,karimireddy2020scaffold,li2020federated}. Equivalently, local copies satisfy $x_i=x_j$ for all $i,j\in[m]$, thereby imposing consensus. This structure \mje{is} unsuitable in heterogeneous systems where agents have different data distributions and operational requirements. In particular, a common model must satisfy every local requirement simultaneously and may therefore be overly conservative, \mje{or such a model may not exist}, while repeatedly replacing local models with a shared global model \mje{prevents agents from retaining personalized adaptations}. Motivated by \mje{these limitations}, we consider a distributed personalized constrained problem \mjee{of the form}
\mje{\begin{align}
\min_{\mathbf{x}}\quad &f(\mathbf{x})\triangleq\tfrac{1}{m}\textstyle\sum_{i=1}^m\left(f_i(\bar{x})+\tfrac{\sigma_i}{2}\|x_i-\bar{x}\|^2\right)\nonumber\\
\text{s.t.}\quad &x_i\in X_i\subseteq\mathbb{R}^n,\qquad \forall i\in[m],\tag{CP-FL}
\label{problem: main problem1}
\end{align}}where $\mathbf{x}\triangleq(x_1,\ldots,x_m)\in\mathbb{R}^{mn}$ is a multi-block decision vector containing one personalized block $x_i\in\mathbb{R}^n$ for each agent $i$, $\bar{x}\triangleq\tfrac{1}{m}\textstyle\sum_{i=1}^m x_i$ is the population average, and $X_i$ is the private feasible set of agent $i$. Moreover, $f_i(\bar{x})\triangleq\mathbb{E}_{\xi_i\sim\mathcal{D}_i}[\tilde f_i(\bar{x},\xi_i)]$, where $\mathcal{D}_i$ is the local data distribution, $\xi_i$ is a random sample, and $\tilde f_i$ is the corresponding sample loss. Unlike consensus-based FL, problem~\eqref{problem: main problem1} does not impose $x_i=x_j$ for all $i,j\in[m]$. Since consensus among the personalized models is not imposed, feasibility does not require the intersection $\cap_{i=1}^m X_i$ to be nonempty. Instead, the multi-block formulation allows each model $x_i$ to converge to a possibly different point in its own feasible set $X_i$. At the same time, the models remain coupled through $\bar{x}$, while $\sigma_i>0$ controls the deviation of $x_i$ from this population average. For example, $X_i$ may represent an agent-specific memory, communication, or model-complexity budget, such as an $\ell_1$-norm constraint with a different sparsity budget for each agent. \mjee{To address problem~\eqref{problem: main problem1},} we propose \mjee{Locally Penalized Cross-Estimate Federated Averaging (PCE-FedAvg).} To understand the main idea, recall that the \mjee{global} objective depends on the population average $\bar{x}$ and therefore couples all personalized models. Accordingly, rather than maintaining only one model, each agent maintains a local copy of the entire collection of personalized models $(x_1,\ldots,x_m)$. For agent $i$, the $i$-th component represents its own personalized model, while the other components are cross-estimates of the models associated with the other agents. Agent $i$ applies a distance-based feasibility correction only to its own component using its private set $X_i$; its cross-estimates are updated without requiring knowledge of the other agents' constraint sets. At each communication round, the server averages all copies corresponding to the same component. Thus, unlike conventional \mjee{methods in FL}, the server does not collapse the different personalized models into a single shared model or force every agent to start from the same model. This block-wise aggregation preserves the server--agent communication structure while allowing each personalized model to approach its own feasible set. Consequently, the method does not require consensus or a nonempty common intersection $\cap_{i=1}^mX_i$.

\textbf{Related Work.} Distributed projection and primal--dual methods have been developed for constrained convex optimization over peer-to-peer networks~\cite{nedic2009distributed,nedic2010constrained,akgun2024projected}. These methods generally seek consensus on a shared decision variable. Constrained FL methods similarly impose global constraints, or locally held constraints, on a common global model and address them through projection, penalty, or augmented-Lagrangian techniques~\cite{yuan2021federated,he2024federated}. Although these approaches accommodate distributed constraint information, they retain the consensus-based shared-model formulation. In contrast, problem~\eqref{problem: main problem1} assigns a distinct feasible block to each agent and does not require the intersection $\cap_{i=1}^mX_i$ to be nonempty. \mjee{Hierarchical FL has also been studied for optimal solution selection using self-tuned regularization~\cite{ebrahimi2026self}, but retains a common model rather than agent-specific models.} Personalized FL addresses statistical heterogeneity through meta-learning, regularized personalized models, shared representations, or hypernetworks~\cite{fallah2020personalized,t2020personalized,li2021ditto,collins2021exploiting,shamsian2021personalized}. However, the standard formulations of these methods do not impose explicit agent-specific feasible sets $X_i$ or provide guarantees on the resulting agent-wise constraint violations. Cross-estimate mechanisms have also been studied in distributed equilibrium seeking~\cite{nguyen2023geometric}, primarily under peer-to-peer communication. Our approach incorporates cross-estimates into server-based FL to coordinate distinct constrained blocks without imposing consensus or disclosing local constraint sets.

\textbf{Contributions.} The main contributions are as follows\mje{:}
\begin{itemize}
\item[-] We introduce a personalized constrained FL formulation that couples agent-specific models through a collaborative objective while allowing each model to satisfy a distinct private constraint without requiring consensus or a common feasible point.

\item[-] We propose \mjee{PCE-FedAvg}, a penalty-based multi-block cross-estimate method that applies feasibility corrections only to each agent's own block \mjee{locally} and uses block-wise server aggregation to preserve personalization \mje{without requiring agents to explicitly share their local constraint sets}.

\item[-] \mjee{For any $\epsilon>0$}, we establish finite-time upper and lower suboptimality bounds and agent-wise squared infeasibility guarantees. The corresponding complexities are $\mathcal{O}(\epsilon^{-2})$ for suboptimality and $\mathcal{O}(\epsilon^{-1})$ for squared infeasibility.

\item[-] Experiments on MNIST and CIFAR-10 with heterogeneous data and agent-specific $\ell_1$ constraints demonstrate the benefits of \mjee{PCE-FedAvg} over penalized variants of FedAvg~\cite{mcmahan2017communication}, FedProx~\cite{li2020federated}, and SCAFFOLD~\cite{karimireddy2020scaffold}.
\end{itemize}

{\bf{\Rme{Notation.}}} For $m\in\mathbb{N}$, let $[m]\triangleq\{1,\dots,m\}$. For vectors $x,y$, define the Euclidean inner product as $\langle x,y\rangle\triangleq x^\top y$. $\|\cdot\|$ denotes the Euclidean norm for vectors and the induced operator (spectral) norm for matrices. For a matrix $M$ and a vector $v$, $M^\top$ and $v^\top$ denote their transposes. For a set $X\subseteq\mathbb{R}^n$, $\Pi_X[\cdot]$ denotes the Euclidean projection onto $X$, and $\mathrm{dist}(x,X)\triangleq \inf_{y\in X}\|x-y\|$ denotes the distance from $x$ to $X$. For sets $\{X_i\}_{i=1}^m$, $\prod_{i=1}^m X_i$ denotes their Cartesian product. \mjee{$I_n$ denotes the identity matrix in $\mathbb{R}^{n\times n}$.}

\section{\Rme{Proposed Algorithm}}

The \mjee{PCE-FedAvg} method is presented in Algorithm~\ref{algorithm: main algorithm-block-wise}. To handle the agent-specific constraints, we consider the penalized stochastic function
\begin{align*}
g_{i,\rho}(\mathbf{x},\xi_i)
\triangleq \tilde f_i(\bar{x},\xi_i)
+\tfrac{\sigma_i}{2}\|x_i-\bar{x}\|^2
+\tfrac{\rho}{2}\operatorname{dist}^2(x_i,X_i).
\end{align*} \mje{At each local iteration $k$, agent $i$ draws $\xi_{i,k}\sim\mathcal{D}_i$, and the samples are assumed independent across agents and iterations.} Since $\nabla_{x_i}\tfrac{1}{2}\operatorname{dist}^2(x_i,X_i)=x_i-\Pi_{X_i}[x_i]$, the penalty term generates a correction that moves the personalized model of agent $i$ toward its local feasible set. To implement these updates without enforcing consensus, each agent maintains a local estimate of the complete multi-block decision vector. Specifically, at iteration $k$, agent $i$ maintains
\begin{align*}
\mathbf{x}_{i,k}\triangleq
\left[(x^{(1)}_{i,k})^\top,\ldots,(x^{(m)}_{i,k})^\top\right]^\top
\in\mathbb{R}^{mn},
\end{align*}
where $x^{(j)}_{i,k}$ is agent $i$'s estimate of the personalized model $x_j$. Thus, $x^{(i)}_{i,k}$ is its own model, whereas $x^{(j)}_{i,k}$, for $j\neq i$, is a cross-estimate of agent $j$'s model. Agent $i$ uses these blocks to form the local estimate $\bar{x}_{i,k}\triangleq\tfrac{1}{m}\textstyle\sum_{j=1}^m x^{(j)}_{i,k}$ of the population average. The two update rules in Algorithm~\ref{algorithm: main algorithm-block-wise} are obtained by taking the corresponding block gradients of $g_{i,\rho}$. The feasibility correction is applied only to $x^{(i)}_{i,k}$ because agent $i$ has access to $X_i$, while its cross-estimates are updated using the objective information without requiring access to $X_j$. At the end of communication round $r$, the server averages all estimates corresponding to the same block:
\begin{align*}
\bar{x}^{(j)}_{r+1}
\triangleq\tfrac{1}{m}\textstyle\sum_{i=1}^m
x^{(j)}_{i,T_{r+1}},\qquad j\in[m].
\end{align*}
Note that, in Algorithm~\ref{algorithm: main algorithm-block-wise}, $\bar{x}_{i,k}$ averages the different blocks within agent $i$'s local multi-block vector, whereas $\bar{x}_{r+1}^{(j)}$ averages the agents' estimates of the same block $j$ at the server. The resulting block averages are then broadcast separately at the beginning of the next round. Unlike conventional FL, the server does not average the different personalized blocks into one shared model. Instead, it maintains one aggregate for each block, allowing the personalized models to differ and approach their respective feasible sets. Consequently, \mjee{PCE-FedAvg} does not require consensus or a nonempty intersection $\cap_{i=1}^mX_i$, and the agents do not disclose their local constraint sets to one another or to the server. No server-side projection or dual variables are required. 
\begin{algorithm}
\caption{\mjee{Locally Penalized Cross-Estimate Federated Averaging (PCE-FedAvg)} }\label{algorithm: main algorithm-block-wise}
\begin{algorithmic}[1]
\State {\textbf{Input:}} Random initial point $\bar{\mathbf{x}}_0\in\mathbb{R}^{mn}$, stepsize $\gamma>0$, penalty parameter $\rho>0$, number of communication rounds $R$ and synchronization indices
$0=T_0<T_1<\cdots<T_R$.
\For{$r = 0,1,\ldots,R-1$}
\State Server sends $\bar{{x}}^{(j)}_r$  to all $i \!\in \![m]$: ${{x}}^{(j)}_{i,T_r} \!:=\! \bar{{{x}}}^{(j)}_r,\forall j\! \in\! [m]$.
\ForAll{$i\in[m]$ \textbf{in parallel}}
\For{$k = T_r,\ldots,T_{r+1}-1$} 
\State \mee{agent} $i$ generates the random variable ${\xi}_{i,k}$.
\State \mee{agent} $i$ evaluates the average of the blocks ${\bar x}_{i,k}=\tfrac{1}{m}\textstyle\sum_{j=1}^{m}{ x}^{(j)}_{i,k}.$
\State \mee{agent} $i$ evaluates $\tilde{g}_{i,k}:=\nabla \tilde f_i({\bar x}_{i,k},\xi_{i,k})$.
\State \mee{agent} $i$ updates its cross-estimates as follows.
\begin{align*}
&{x}^{(i)}_{i,k+1} := { x}^{(i)}_{i,k} - {\gamma}\left(\tfrac{1}{m}\tilde{g}_{i,k} + \rho (x^{(i)}_{i,k} - \Pi_{X_i}[x^{(i)}_{i,k}])\right.\\
&\left.\qquad\qquad\qquad\quad+ \sigma_i(\tfrac{m-1}{m})({ x}^{(i)}_{i,k} - {\bar x}_{i,k})\right), \\
&{x}^{(j)}_{i,k+1} := { x}^{(j)}_{i,k} - {\gamma}(\tfrac{1}{m}\tilde{g}_{i,k}- \tfrac{\sigma_i}{m}({ x}^{(i)}_{i,k} - {\bar x}_{i,k})),\ \hbox{for }j\neq i.
\end{align*}
\EndFor
\EndFor
\State Server receives ${x}^{(j)}_{i,T_{r+1}}$ from all \mee{agent}s for all $j \in [m]$ and aggregates $\bar{{x}}^{(j)}_{r+1} := \tfrac{1}{m}\textstyle\sum_{i=1}^{m} {x}^{(j)}_{i,T_{r+1}}$, $\forall j \in [m]$.
\EndFor
\end{algorithmic}
\end{algorithm}

\mje{\noindent\textbf{History of the method.}
For $k\geq1$, define the history available before iteration $k$ as $\mathcal{F}_k\triangleq\{\bar{\mathbf{x}}_0\}\cup\bigcup_{i=1}^m\bigcup_{t=0}^{k-1}\{\xi_{i,t}\},$ and let $\mathcal{F}_0\triangleq\{\bar{\mathbf{x}}_0\}$.}

\section{Methodology}\label{section: Methodology}

This section presents the penalized multi-block formulation of problem~\eqref{problem: main problem1}. For each agent $i\in[m]$, we measure its local constraint violation using
\begin{align}
h_i(x_i)\triangleq\tfrac{1}{2}\operatorname{dist}^2(x_i,X_i)
=\tfrac{1}{2}\left\|x_i-\Pi_{X_i}[x_i]\right\|^2.
\end{align}
Thus, $h_i(x_i)\geq0$, with $h_i(x_i)=0$ if and only if $x_i\in X_i$.

We define the global multi-block decision vector as $\mathbf{x}\triangleq[x_1^\top,\ldots,x_m^\top]^\top\in\mathbb{R}^{mn}$. Let $A\triangleq\tfrac{1}{m}\begin{bmatrix}I_n&I_n&\cdots&I_n\end{bmatrix}\in\mathbb{R}^{n\times mn}$, so that $A\mathbf{x}=\bar{x}$. For each agent $i$, define the block-selection matrix $B_i\triangleq\begin{bmatrix}0&\cdots&0&I_n&0&\cdots&0\end{bmatrix}\in\mathbb{R}^{n\times mn}$, where the identity block occupies the $i$-th position, so that $B_i\mathbf{x}=x_i$. Under this representation, define
\begin{align*}
f(\mathbf{x})&\triangleq\tfrac{1}{m}\textstyle\sum_{i=1}^m
\mathbb{E}\left[\tilde F_i(\mathbf{x},\xi_i)\right],\quad
H(\mathbf{x})\triangleq\tfrac{1}{m}\textstyle\sum_{i=1}^mH_i(\mathbf{x}),
\end{align*}
where $\tilde F_i(\mathbf{x},\xi_i)\triangleq\tilde f_i(A\mathbf{x},\xi_i)+\tfrac{\sigma_i}{2}\|B_i\mathbf{x}-A\mathbf{x}\|^2$ and $H_i(\mathbf{x})\triangleq h_i(B_i\mathbf{x})$. For a penalty parameter $\rho>0$, define the local stochastic penalized function as $g_{i,\rho}(\mathbf{x},\xi_i)\triangleq\tilde F_i(\mathbf{x},\xi_i)+\rho H_i(\mathbf{x})$ and its expectation as $G_{i,\rho}(\mathbf{x})\triangleq\mathbb{E}[g_{i,\rho}(\mathbf{x},\xi_i)]$. The global penalized objective is $G_\rho(\mathbf{x})\triangleq\tfrac{1}{m}\textstyle\sum_{i=1}^mG_{i,\rho}(\mathbf{x})=f(\mathbf{x})+\rho H(\mathbf{x})$. The penalty term promotes agent-wise feasibility while preserving the coupling among the personalized models through $A\mathbf{x}=\bar{x}$. Using this notation, the local update in Algorithm~\ref{algorithm: main algorithm-block-wise} can be written compactly as
\begin{align}
\mathbf{x}_{i,k+1}
=\mathbf{x}_{i,k}
-\gamma\nabla_{\mathbf{x}}g_{i,\rho}(\mathbf{x}_{i,k},\xi_{i,k}).
\tag{CR}\label{equation: compact update rule}
\end{align}
Finally, let $\mathbf{x}^*$ denote an optimal solution of the original constrained problem~\eqref{problem: main problem1}, whereas $\mathbf{x}_\rho^*\in\operatorname*{argmin}_{\mathbf{x}\in\mathbb{R}^{mn}}G_\rho(\mathbf{x})$ denotes a minimizer of the penalized objective. Although \mjee{PCE-FedAvg} performs stochastic-gradient updates on $G_\rho$, our analysis controls the error introduced by penalization and relates the penalized objective gap to both the original objective gap and agent-wise infeasibility. Consequently, the final convergence and communication-complexity guarantees are stated with respect to the optimal solution $\mathbf{x}^*$ of problem~\eqref{problem: main problem1}, rather than the penalized solution $\mathbf{x}_\rho^*$. These guarantees are summarized in the next section.

\section{Main Results}\label{section: Convergence Analysis}
In this section, we analyze Algorithm~\ref{algorithm: main algorithm-block-wise}. We first state the standard assumptions and then present the main results for solving problem~\eqref{problem: main problem1}.
\begin{assumption}\em\label{assumption: main assumption 1}
For each \mee{agent} $i\in \fy{[m]}$, assume that $f_i$ is $L_f$-smooth and convex. \fy{Further, each \mee{agent} has access to an unbiased local stochastic gradient oracle, ensuring 
$\mathbb{E}_{\xi_i}\![ \nabla \tilde f_i(x,\xi_{i,k})\mid x]=\nabla f_i(x)$, for $k\geq 0$, and all $x \in \mathbb{R}^n$.}

\end{assumption}
\begin{assumption}\em \label{assumption: main assumption 3}  \fy{There} exists $\nu \ge 0$ such that \fy{for all $i\in[m]$, $\mathbb{E}_{\xi_i}\![\|\nabla \tilde f_i(x,\xi_{i,k})-\nabla f_i(x) \|^2\mid x]\le \nu^2$, for all $x \in \mathbb{R}^n$.}
\end{assumption}

\begin{assumption}\em\label{assumption: main assumption 2}
For each $i\in[m]$, the set $X_i\subseteq\mathbb{R}^n$ is nonempty, closed, and convex. 
\end{assumption}

\fy{The main results of this work are presented in} the following theorem, which provides communication-complexity guarantees for lower and upper bounds on the suboptimality of the global objective, as well as for \mee{agent}-wise infeasibility. 
\begin{theorem}\em\label{theorem: best rho}
Consider Algorithm~\ref{algorithm: main algorithm-block-wise} under Assumptions~\ref{assumption: main assumption 1},~\ref{assumption: main assumption 3}, and~\ref{assumption: main assumption 2}. Define $\sigma\triangleq\max_i\sigma_i$, $L\triangleq L_f+\sigma m$, and $L_{G_\rho}\triangleq\tfrac{L_f+\sigma(m-1)}{m}+\rho$. Let $m\geq2$, $H\geq2$, and $T_{r+1}-T_r=H$ for $r=0,\ldots,R-1$. Set $\rho\triangleq\sqrt{R}$ and $\gamma\triangleq\tfrac{\gamma_0}{\sqrt{R}}$, where $\gamma_0>0$ is independent of $R$ and sufficiently small to ensure $0<\gamma\leq\min\{\tfrac{1}{6L_{G_\rho}},\tfrac{1}{5L_{G_\rho}(H-1)}\}$. Assume that $\mathbf{x}^*\triangleq[(x_1^*)^\top,\ldots,(x_m^*)^\top]^\top$ is an optimal solution of problem~\eqref{problem: main problem1}. For every $\rho\geq1$, let $\mathbf{x}_\rho^*\in\operatorname*{argmin}_{\mathbf{x}}G_\rho(\mathbf{x})$, and suppose that \mje{$Q\triangleq\sup_{\rho\geq1}\mathbb{E}[\|\bar{\mathbf{x}}_0-\mathbf{x}_\rho^*\|^2]$, $M_f\triangleq\sup_{\rho\geq1}\frac{1}{m}\textstyle\sum_{i=1}^m\|\nabla f_i(A\mathbf{x}_\rho^*)\|^2$, and $M_\sigma\triangleq\sup_{\rho\geq1}\frac{1}{m}\textstyle\sum_{i=1}^m\sigma_i^2\|B_i\mathbf{x}_\rho^*-A\mathbf{x}_\rho^*\|^2$} are finite. Define $D\triangleq\|\nabla f(\mathbf{x}^*)\|^2$, $K\triangleq RH$, and $\mje{\widehat{\mathbf{x}}_K
\triangleq \tfrac{1}{K}\textstyle\sum_{k=0}^{K-1}\bar{\mathbf{x}}_k.}$ Then, the following hold.

\noindent(i){ \bf{[Suboptimality-Upper Bound]}} Let $\epsilon >0$. To ensure that $\mathbb{E}\left[f(\widehat{\mathbf{x}}_K)-f(\mathbf{x}^*)\right]\le \epsilon$, it suffices to take
\me{\begin{align*}R\!=&\mathcal{O}\!\left( \tfrac{Q^2}{H^2\epsilon^2}\!+\tfrac{H^4\mje{(M_f+mM_\sigma)}^2}{\epsilon^2}+\tfrac{LH^{2}\nu^2}{m^2\epsilon}+\tfrac{H^4\nu^4}{m^2\epsilon^2}\right).
\end{align*}}
\noindent(ii){ \bf{[Feasibility]}} Given any $\epsilon>0$, the agent-wise feasibility requirement $\mathbb{E}\!\left[\left\| \widehat x_K^{(i)} - \Pi_{X_i}\!\left[\widehat x_K^{(i)}\right]\right\|^2\right]\le \epsilon$ for every $i\in[m]$ is guaranteed whenever
\begin{align*}
R = &\mathcal{O}\!\left(\!\tfrac{m^2D}{\epsilon}+{{\tfrac{mQ}{H\epsilon}}}\!+\!{\tfrac{{L}^{{2}/{3}} H^{{4}/{3}}\mje{(M_f+mM_\sigma)}^{{2}/{3}}}{\epsilon^{{2}/{3}}}}+\tfrac{{mH^2\mje{(M_f+mM_\sigma)}}{}}{\epsilon}+\!{\tfrac{{L}^{{2}/{3}} H^{{4}/{3}}\nu^{{4}/{3}}}{m^{{2}/{3}}\epsilon^{{2}/{3}}}}\!+\!\tfrac{{H^2\nu^2}{}}{\epsilon} \right).
\end{align*}
\noindent(iii){ \bf{[Suboptimality-Lower Bound]}} Let $\epsilon>0$. The expected suboptimality is bounded below by $-\epsilon$ if
\begin{align*}
R=&\mathcal{O}\left(\tfrac{m^2D^2}{\epsilon^2}+\tfrac{mDQ}{H\epsilon^2}+\tfrac{L^{2/3}D^{2/3}H^{4/3}\mje{(M_f+mM_\sigma)}^{2/3}}{\epsilon^{4/3}}+\tfrac{mDH^2\mje{(M_f+mM_\sigma)}}{\epsilon^2}+\tfrac{L^{2/3}D^{2/3}H^{4/3}\nu^{4/3}}{m^{2/3}\epsilon^{4/3}}+\tfrac{DH^2\nu^2}{\epsilon^2}\right).
\end{align*}
\end{theorem}
\begin{remark}\em
Theorem~\ref{theorem: best rho} provides communication-complexity guarantees for solving problem~\eqref{problem: main problem1}. Unlike conventional constrained FL methods, which retain a single shared model, our formulation permits heterogeneous agent-specific models and does not enforce consensus. Nevertheless, \mjee{PCE-FedAvg} achieves the standard $\mathcal{O}(\epsilon^{-2})$ communication complexity for suboptimality while attaining $\mathcal{O}(\epsilon^{-1})$ communication complexity for agent-wise squared infeasibility. Thus, personalization does not degrade the order-wise suboptimality complexity.
\end{remark}

\begin{remark}\em
Since $\mathbf{x}^*$ minimizes $f$ only over the constraint set, an infeasible iterate may have an objective value below $f(\mathbf{x}^*)$. Therefore, both upper and lower bounds on the objective gap are needed. The penalty parameter balances suboptimality and feasibility: increasing $\rho$ promotes feasibility but may worsen the suboptimality bound. The choice $\rho\triangleq\sqrt{R}$ balances these effects and yields $\mathcal{O}(\epsilon^{-2})$ communication complexity for suboptimality and $\mathcal{O}(\epsilon^{-1})$ for agent-wise squared infeasibility.
\end{remark}
%
%
\section{Convergence Analysis}
In this section, we present supporting lemmas and propositions that will be used to establish Theorem~\ref{theorem: best rho}.
First, we establish basic properties of the penalized objective. 
\begin{lemma}\label{lemma:Gi_smooth}\em
Suppose Assumptions~\ref{assumption: main assumption 1} and~\ref{assumption: main assumption 2} hold. 
Then, for each $i\in[m]$, $G_{i,\rho}$ is $L_{G_{i,\rho}}$-smooth on $\mathbb{R}^{mn}$ with \me{$L_{G_{i,\rho}} \triangleq L_f\|A\|^2 +\sigma_i\| B_i-A\|^2+\rho\,\|B_i\|^2$}, where $\rho>0$ is a penalty parameter. In particular, since \me{$\|B_i\|^2=1$, $\| B_i-A\|^2=\frac{m-1}{m}$, for all $i$, and $\|A\|^2=\tfrac{1}{m}$}, we have \me{$L_{G_{i,\rho}}\le L_{G_\rho} \triangleq \frac{L_f+\sigma(m-1)}{m} + \rho$, where $\sigma=\max_i \sigma_i$}.
\end{lemma}

\begin{lemma}\em\label{lemma:Gi_convex}
Suppose Assumptions~\ref{assumption: main assumption 1} and~\ref{assumption: main assumption 2} hold. Then, for each $i\in[m]$, $G_{i,\rho}$ is convex on $\mathbb{R}^{mn}$.
\end{lemma}
For ease of analysis, we define $\bar{d}_i \triangleq \|\widehat x_K^{(i)}-\Pi_{X_i}[\widehat x_K^{(i)}]\|,
\bar d \triangleq \|\widehat{\mathbf{x}}_K-\Pi_X[\widehat{\mathbf{x}}_K]\|, \delta_{i,k} \!\triangleq  \!\| {\mathbf{x}}_{i,k}\!-\!\bar{\mathbf{x}}_k \|^2, \  \widehat{x}_K^{(i)}\!\!\triangleq\!\tfrac{1}{K}\textstyle\sum_{k=0}^{K-1}\bar{x}_k^{(i)}, \text{and } \bar{\mathbf{x}}_k\!\!\triangleq\!\tfrac{1}{m}\!\sum_{i=1}^m\!\mathbf{x}_{i,k}$. Next, we present supporting lemmas that will be used in the subsequent analysis.

\mje{\begin{lemma}\em\label{lemma: local gradients at penalized solution}
For every $\rho\geq1$, let $\mathbf{x}_\rho^*\in
\operatorname*{argmin}_{\mathbf{x}}G_\rho(\mathbf{x})$. Suppose the
uniform bounds $M_f$ and $M_\sigma$ defined in
Theorem~\ref{theorem: best rho} are finite. Then,
\begin{align*}
\textstyle\frac{1}{m}\sum_{i=1}^m
\left\|\nabla G_{i,\rho}(\mathbf{x}_\rho^*)\right\|^2\leq 8M_f+4mM_\sigma.
\end{align*}
\end{lemma}

\begin{proof}
Define $r_i\triangleq B_i\mathbf{x}_\rho^*-\Pi_{X_i}[B_i\mathbf{x}_\rho^*]$. Since $\mathbf{x}_\rho^*$ is an unconstrained minimizer of $G_\rho$, $0=\nabla f(\mathbf{x}_\rho^*)+\frac{\rho}{m}\sum_{j=1}^mB_j^\top r_j$. Using $B_iB_i^\top=I_n$ and $B_iB_j^\top=0$ for $i\neq j$, we obtain $\rho B_i^\top r_i=-mB_i^\top B_i\nabla f(\mathbf{x}_\rho^*)$, and hence $\nabla G_{i,\rho}(\mathbf{x}_\rho^*)=\nabla F_i(\mathbf{x}_\rho^*)-mB_i^\top B_i\nabla f(\mathbf{x}_\rho^*)$. Using $\|A\|^2=1/m$ and $\|B_i-A\|^2=(m-1)/m$, we have $\tfrac{1}{m}\textstyle\sum_{i=1}^m\|\nabla F_i(\mathbf{x}_\rho^*)\|^2
\leq\frac{2M_f}{m}+\frac{2(m-1)M_\sigma}{m}$. Since $\nabla f=\frac{1}{m}\sum_i\nabla F_i$, Jensen's inequality gives the same bound for $\|\nabla f(\mathbf{x}_\rho^*)\|^2$. Moreover, $\sum_i\|B_i^\top B_i z\|^2=\|z\|^2$. Substituting the preceding bounds into
$\tfrac{1}{m}\sum_{i=1}^m\|\nabla G_{i,\rho}(\mathbf{x}_\rho^*)\|^2
\leq\tfrac{2}{m}\sum_{i=1}^m\|\nabla F_i(\mathbf{x}_\rho^*)\|^2
+2m\|\nabla f(\mathbf{x}_\rho^*)\|^2$
yields the result.
\end{proof}}
\begin{lemma}\em\label{Lemma: the bound for nabla G_rho(mathbf{x}_k)} Consider Algorithm~\ref{algorithm: main algorithm-block-wise} and let Assumptions~\ref{assumption: main assumption 1},~\ref{assumption: main assumption 3}, and~\ref{assumption: main assumption 2} hold. Then, for $m\ge 2$, we have 
\begin{align*}
&\mathbb{E}[\|\tfrac{1}{m}\textstyle\sum_{i=1}^m\!\nabla \! g_{i,\rho}(\mathbf{x}_{i,k},\xi_{i,k})\|^2] \le 6L_{G_\rho}\!\mathbb{E}[  G_{\rho}(\bar{\mathbf{x}}_k)\!-\!G_{\rho}(\mathbf{x}_\rho^*) ]+ {3L_{G_\rho}^2}\mathbb{E} [\tfrac{1}{m}\textstyle\sum_{i=1}^m \delta_{i,k} ]+{3\nu^2}/{m}.
\end{align*}
\end{lemma}
\begin{proof}
First, we start with $\mathbb{E}[\|\frac{1}{m}\sum_{i=1}^m\nabla g_{i,\rho}(\mathbf{x}_{i,k},\xi_{i,k})\|^2]$. Adding and subtracting $\frac{1}{m}\sum_{i=1}^m \nabla G_{i,\rho}(\bar{\mathbf{x}}_k)$ and $\frac{1}{m}\sum_{i=1}^m \nabla G_{i,\rho}({\mathbf{x}}_{i,k})$, yields
\begin{align*}
&\mathbb{E}  [\| \tfrac{1}{m} \textstyle\sum_{i=1}^m \nabla g_{i,\rho}(\mathbf{x}_{i,k},\xi_{i,k})\|^2]  \leq 3\mathbb{E}  [\| \tfrac{1}{m} \textstyle\sum_{i=1}^m   \nabla G_{i,\rho}(\bar{\mathbf{x}}_k)\|^2]\\ &+ 3\mathbb{E} [\|\tfrac{1}{m}\textstyle\sum_{i=1}^m \nabla g_{i,\rho}(\mathbf{x}_{i,k},\xi_{i,k}) - \tfrac{1}{m}\sum_{i=1}^m  \nabla G_{i,\rho}({\mathbf{x}}_{i,k})\|^2]\\
&+3\mathbb{E}[\|\tfrac{1}{m}\textstyle\sum_{i=1}^m \nabla G_{i,\rho}({\mathbf{x}}_{i,k})-\tfrac{1}{m}\sum_{i=1}^m \nabla G_{i,\rho}(\bar{\mathbf{x}}_k)\|^2].
\end{align*}
Invoking Lemma~\ref{lemma:Gi_smooth}, we obtain
\begin{align*}
&\mathbb{E}[\|\tfrac{1}{m}\textstyle\sum_{i=1}^m\nabla g_{i,\rho}(\mathbf{x}_{i,k},\xi_{i,k})\|^2]\le3\mathbb{E}[\| \nabla G_{\rho}(\bar{\mathbf{x}}_k)\|^2] +\tfrac{3}{m}\textstyle\sum_{i=1}^m\mathbb{E}[\|\nabla g_{i,\rho}(\mathbf{x}_{i,k},\xi_{i,k})- \nabla G_{i,\rho}({\mathbf{x}}_{i,k})\|^2]\\
& +\tfrac{3}{m}\textstyle\sum_{i=1}^m\mathbb{E}[\|\nabla G_{i,\rho}({\mathbf{x}}_{i,k})- \nabla G_{i,\rho}(\bar{\mathbf{x}}_k)\|^2]\le\tfrac{3}{m}\textstyle\sum_{i=1}^m\mathbb{E}[\|A^\top \nabla \tilde f_i(\bar x,\xi_i) -  A^\top \nabla  f_i(\bar x)\|^2]\\
& + {3L_{G_\rho}^2}\mathbb{E}[\tfrac{1}{m}\textstyle\sum_{i=1}^m\delta_{i,k}]+3\mathbb{E}[\| \nabla G_{\rho}(\bar{\mathbf{x}}_k)\|^2].
\end{align*}
Invoking Assumption~\ref{assumption: main assumption 3}, we obtain
\begin{align}
&\mathbb{E}[\|\tfrac{1}{m}\textstyle\sum_{i=1}^m\nabla g_{i,\rho}(\mathbf{x}_{i,k},\xi_{i,k})\|^2]\me{\le}\tfrac{3}{m}\textstyle\sum_{i=1}^m\|A\|^2\mathbb{E}[\| \nabla \tilde f_i(\bar x,\xi_i)- \nabla  f_i(\bar x)\|^2]\label{equation: updated version in the first lemma}\\
& + {3L_{G_\rho}^2}\mathbb{E}[\tfrac{1}{m}\textstyle\sum_{i=1}^m\delta_{i,k}]+3\mathbb{E}[\| \nabla G_{\rho}(\bar{\mathbf{x}}_k)\|^2]\le  {3\nu^2}/{m} +  {3L_{G_\rho}^2} \mathbb{E} [ \tfrac{1}{m} \textstyle\sum_{i=1}^m \delta_{i,k}]  + 3\mathbb{E} [\| \nabla G_{\rho}(\bar{\mathbf{x}}_k)\|^2].\notag
\end{align}
Since $\nabla G_{\rho}(\mathbf{x}_\rho^*) = 0$, we add $-\nabla G_{\rho}(\mathbf{x}_\rho^*)$ to the last term and invoke Lemmas~\ref{lemma:Gi_smooth} and~\ref{lemma:Gi_convex} for the last term, we obtain $3\mathbb{E}[\| \nabla G_{\rho}(\bar{\mathbf{x}}_k)-\nabla G_{\rho}(\mathbf{x}_\rho^*)\|^2]
\me{\le} 6L_{G_\rho}\mathbb{E}[  G_{\rho}(\bar{\mathbf{x}}_k) - G_{\rho}(\mathbf{x}_\rho^*) ] - 6L_{G_\rho}\left\langle \nabla G_{\rho}(\mathbf{x}_\rho^*),\bar{\mathbf{x}}_k-\mathbf{x}_\rho^*\right\rangle$. Invoking the fact that $\nabla G_{\rho}(\mathbf{x}_\rho^*) = 0$, we obtain $3\mathbb{E} [\| \nabla G_{\rho}(\bar{\mathbf{x}}_k) - \nabla G_{\rho}(\mathbf{x}_\rho^*)\|^2]  \le  6L_{G_\rho} \mathbb{E}  [  G_{\rho}(\bar{\mathbf{x}}_k) - G_{\rho}(\mathbf{x}_\rho^*)]$. Substituting into~\eqref{equation: updated version in the first lemma} yields the desired result.
\end{proof}

\begin{lemma}\em \label{Lemma: bound for inner product term}
Suppose that Assumptions~\ref{assumption: main assumption 1}, and~\ref{assumption: main assumption 2} hold and consider Algorithm~\ref{algorithm: main algorithm-block-wise}. Then the following holds.
\begin{align*}
&-\tfrac{2}{m}\textstyle\sum_{i=1}^{m}\left\langle \bar{\mathbf{x}}_{k}-\mathbf{x}_\rho^{*},\,\nabla G_{i,\rho}(\mathbf{x}_{i,k})\right\rangle\me{\le}-2\left(G_{\rho}(\bar{\mathbf{x}}_{k})-G_{\rho}(\mathbf{x}_\rho^{*})\right)
+\tfrac{L_{G_\rho}}{m}\textstyle\sum_{i=1}^{m}\delta_{i,k}.
\end{align*}
\end{lemma}
\begin{proof}
Starting with the left-hand side, we have $-2\big\langle \bar{\mathbf{x}}_{k}-\mathbf{x}_\rho^{*},\nabla G_{i,\rho}(\mathbf{x}_{i,k})\big\rangle\me{=} -2\big\langle \mathbf{x}_{i,k}-\mathbf{x}_\rho^{*},\nabla G_{i,\rho}(\mathbf{x}_{i,k})\big\rangle    - 2\big\langle \bar{\mathbf{x}}_{k}-\mathbf{x}_{i,k},\nabla G_{i,\rho}(\mathbf{x}_{i,k})\big\rangle$. Invoking Lemmas~\ref{lemma:Gi_smooth} and~\ref{lemma:Gi_convex}, we obtain $-2\big\langle \bar{\mathbf{x}}_{k}-\mathbf{x}_\rho^{*},\nabla G_{i,\rho}(\mathbf{x}_{i,k})\big\rangle \le 2\left(G_{i,\rho}(\mathbf{x}_\rho^{*})-G_{i,\rho}(\mathbf{x}_{i,k})\right)+ 2(G_{i,\rho}(\mathbf{x}_{i,k})-G_{i,\rho}(\bar{\mathbf{x}}_{k}))
 +{L_{G_\rho}}{}\delta_{i,k}=2(G_{i,\rho}(\mathbf{x}_\rho^{*})-G_{i,\rho}(\bar{\mathbf{x}}_{k}))
+{L_{G_\rho}}{}\delta_{i,k}$. By averaging across all $i$, we obtain the desired result.
\end{proof}
\begin{lemma}\em\label{lemma: lemma for bounding the consensus} Consider Algorithm~\ref{algorithm: main algorithm-block-wise}. Assume that Assumptions~\ref{assumption: main assumption 1},~\ref{assumption: main assumption 3}, and~\ref{assumption: main assumption 2} hold. Let $T_{r+1}-T_r= H$ and $v=T_{r+1}-1$, for any $r\ge 0$, and $\gamma\le \frac{1}{5L_{G_\rho}(H-1)}$, then 
\begin{align*}
&\textstyle\sum_{k=T_r}^{v}
\! \! \mathbb{E} [ \tfrac{1}{m}\! \textstyle\sum_{i=1}^{m}  \delta_{i,k} ] \le\! 5\gamma^{2}(H\! -\!  1)^{2}\textstyle\! \sum_{t=T_r}^{v} \left(\mje{4mM_\sigma{}\!+\!8{M_f}{}}+  {\nu^2}/{m}\right)\\
&\fy{+} 10\gamma^{2}(H-1)^{2}L_{G_\rho}\textstyle\sum_{t=T_r}^{v}\mathbb{E}[(G_{\rho}(\bar{\mathbf{x}}_t)- G_{\rho}({\mathbf{x}}_\rho^*))].
\end{align*}

\end{lemma}
\begin{proof}
Let $T_r \le k \le T_{r+1}-1 = v$. From ~\eqref{equation: compact update rule}, recursively expanding the iterates and using $\bar{ \mathbf{x}}_{T_r}=\mathbf{x}_{i,T_r}$ for all $i$, yields
\begin{align}
&\mathbb{E} [\tfrac{1}{m}\textstyle\sum_{i=1}^{m}\delta_{i,k}] = \tfrac{\gamma^{2}}{m} \textstyle\sum_{i=1}^{m}  \mathbb{E} [
 \|\textstyle\sum_{t=T_r}^{k-1}   (\nabla g_{i,\rho}(\mathbf{x}_{i,t},\xi_{i,t})  - \tfrac{1}{m} \textstyle\sum_{j=1}^{m}  \nabla g_{j,\rho}(\mathbf{x}_{j,t},\xi_{j,t}) )\|^{2}]\notag\\
 &  \le\!  \tfrac{\gamma^{2}(k-T_r)}{m} \textstyle\sum_{i=1}^{m} \sum_{t=T_r}^{k-1} \mathbb{E} [\|\nabla g_{i,\rho}(\mathbf{x}_{i,t},\xi_{i,t}) \!- \tfrac{1}{m} \textstyle\sum_{j=1}^{m}  \nabla g_{j,\rho}(\mathbf{x}_{j,t},\xi_{j,t})\|^{2}].\label{equation: inequality containing the term where it is bounded incorrectly}
\end{align}
We know that
\begin{align*}
&\textstyle\sum_{i=1}^{m} \mathbb{E}  [\big\|\nabla g_{i,\rho}(\mathbf{x}_{i,t},\xi_{i,t})  - \tfrac{1}{m} \textstyle\sum_{j=1}^{m}  \nabla g_{j,\rho}(\mathbf{x}_{j,t},\xi_{j,t})\big\|^{2}]   =  \textstyle\sum_{i=1}^{m} \mathbb{E} [\big\|\nabla g_{i,\rho}(\mathbf{x}_{i,t},\xi_{i,t})\big\|^{2}]  \\
& -m\mathbb{E}\left [\big\|\tfrac{1}{m} \textstyle\sum_{i=1}^{m} \nabla g_{i,\rho}(\mathbf{x}_{i,t},\xi_{i,t})\big\|^{2}\right] \le \textstyle\sum_{i=1}^{m}\mathbb{E}[\big\|\nabla g_{i,\rho}(\mathbf{x}_{i,t},\xi_{i,t})\big\|^{2}].
\end{align*}
Substituting the preceding bound into~\eqref{equation: inequality containing the term where it is bounded incorrectly} and summing over $k$ from $T_r$ to $v$, we obtain
\begin{align}
&\textstyle\sum_{k=T_r}^{v}\mathbb{E} [\frac{1}{m}\sum_{i=1}^{m}\delta_{i,k}]\notag\\
&\me{\le}\textstyle\sum_{k=T_r+1}^{v}\tfrac{\gamma^{2}(H-1)}{m}\sum_{i=1}^{m}\sum_{t=T_r}^{k-1}\mathbb{E} [\big\|\nabla g_{i,\rho}(\mathbf{x}_{i,t},\xi_{i,t})\big\|^{2}]\notag\\
&\me{\le}\tfrac{\gamma^{2}(H-1)^{2}}{m}\textstyle\sum_{i=1}^{m}\sum_{t=T_r}^{v}\mathbb{E}[\big\|\nabla g_{i,\rho}(\mathbf{x}_{i,t},\xi_{i,t})\big\|^{2}].\label{equation:After bounding k-1 by v-1}
\end{align}
To bound $\frac{1}{m}\sum_{i=1}^{m}\mathbb{E}[\big\|\nabla g_{i,\rho}(\mathbf{x}_{i,t},\xi_{i,t})\big\|^{2}]$, we add and subtract $\nabla G_{i,\rho}(\bar{\mathbf{x}}_{t})$, $\nabla G_{i,\rho}(\mathbf{x}_{\rho}^{*})$, and $\nabla G_{i,\rho}(\mathbf{x}_{i,t})$, yielding
\begin{align}
&\tfrac{1}{m} \textstyle\sum_{i=1}^{m} \mathbb{E} [\|\nabla g_{i,\rho}(\mathbf{x}_{i,t},\xi_{i,t})\|^{2}]
  \le  \tfrac{4}{m} \textstyle\sum_{i=1}^{m} \mathbb{E} [\|\nabla G_{i,\rho}(\mathbf{x}_{\rho}^{*})\|^{2}]\notag\\
&+\tfrac{4}{m}\textstyle\sum_{i=1}^{m}\mathbb{E}[\|\nabla g_{i,\rho}(\mathbf{x}_{i,t},\xi_{i,t})
-\nabla G_{i,\rho}(\mathbf{x}_{i,t})\|^{2}] \notag\\
&\ +\tfrac{4}{m}\textstyle\sum_{i=1}^{m}\mathbb{E}[\|\nabla G_{i,\rho}(\bar{\mathbf{x}}_{t})
-\nabla G_{i,\rho}(\mathbf{x}_{\rho}^{*})\|^{2}]+\tfrac{4}{m}\textstyle\sum_{i=1}^{m}\mathbb{E}[\|\nabla G_{i,\rho}(\mathbf{x}_{i,t})-\nabla G_{i,\rho}(\bar{\mathbf{x}}_{t})
\|^{2}].\label{equation: bounding the grad by adding and subtracting two terms}
\end{align}
By Lemmas~\ref{lemma:Gi_smooth} and~\ref{lemma:Gi_convex}, $G_{i,\rho}$ is $L_{G_\rho}$-smooth and convex. Invoking the fact that $\tfrac{1}{m}\textstyle\sum_{i=1}^m\nabla G_{i,\rho}({\mathbf{x}}_\rho^*)=0$, we obtain
\begin{align*}&\tfrac{1}{m}\textstyle\sum_{i=1}^m\mathbb{E}[\|\nabla G_{i,\rho}(\bar{\mathbf{x}}_t)-\nabla G_{i,\rho}({\mathbf{x}}_\rho^*)\|^2]\me{\le}\tfrac{2L_{G_\rho}}{m}\textstyle\sum_{i=1}^m\mathbb{E}[G_{i,\rho}(\bar{\mathbf{x}}_t)- G_{i,\rho}({\mathbf{x}}_\rho^*)]\\
&-{2L_{G_\rho}}\mathbb{E}[\left\langle\tfrac{1}{m}\textstyle\sum_{i=1}^m\nabla G_{i,\rho}({\mathbf{x}}_\rho^*),\bar{\mathbf{x}}_t-\mathbf{x}_\rho^*\right\rangle]={2L_{G_\rho}}\mathbb{E}[\left(G_{\rho}(\bar{\mathbf{x}}_t)- G_{\rho}({\mathbf{x}}_\rho^*)\right)].
\end{align*}
Incorporating the preceding bound into~\eqref{equation: bounding the grad by adding and subtracting two terms}, and invoking Lemmas~\ref{lemma:Gi_smooth} and~\ref{lemma: local gradients at penalized solution}, we obtain
\begin{align*}
&\tfrac{1}{m}\textstyle\sum_{i=1}^{m}
\mathbb{E}[\|\nabla g_{i,\rho}(\mathbf{x}_{i,t},\xi_{i,t})\|^{2}]\me{\le}4L_{G_\rho}^{2}  \mathbb{E} [\tfrac{1}{m} \textstyle\sum_{i=1}^{m}\|\mathbf{x}_{i,t} - \bar{\mathbf{x}}_{t}\|^{2}]  \\ & +\!  8L_{G_\rho}\mathbb{E} [G_{\rho}(\bar{\mathbf{x}}_t) \!-  \!G_{\rho}({\mathbf{x}}_\rho^*)]\!+\!\!\mje{16mM_\sigma{}\!+\!32{M_f}{}} +\!\tfrac{4}{m}\textstyle\sum_{i=1}^{m}\mathbb{E}[\|A^\top \nabla \tilde f_i(\bar x,\xi_i)-  A^\top \nabla  f_i(\bar x)\|^2].\end{align*}
By Assumption~\ref{assumption: main assumption 3} and $\|A\|^2 = \tfrac{1}{m}$, we obtain
\begin{align*}
&\tfrac{1}{m}  \textstyle\sum_{i=1}^{m}  \mathbb{E}  [\|\nabla g_{i,\rho}(\mathbf{x}_{i,t},\xi_{i,t})\|^{2}]   \le  8L_{G_\rho} \mathbb{E} [\left(G_{\rho}(\bar{\mathbf{x}}_t) -  G_{\rho}({\mathbf{x}}_\rho^*)\right)]+{4\nu^2}/{m}\\
&+4L_{G_\rho}^{2}\mathbb{E}[\tfrac{1}{m}\textstyle\sum_{i=1}^{m}\delta_{i,t}]+\mje{16mM_\sigma{}\!+\!32{M_f}{}}.
\end{align*}
Substituting the preceding bound into~\eqref{equation:After bounding k-1 by v-1} and rearranging the terms, we obtain
\begin{align*}
&(1 - 4\gamma^{2}(H - 1)^{2}L_{G_\rho}^{2}) \textstyle\sum_{k=T_r}^{v} \mathbb{E}[\tfrac{1}{m}\sum_{i=1}^{m}\delta_{i,k}]\\
&\me{\le}8\gamma^{2}(H-1)^{2}L_{G_\rho}\textstyle\sum_{t=T_r}^{v}\mathbb{E}[\left(G_{\rho}(\bar{\mathbf{x}}_t)- G_{\rho}({\mathbf{x}}_\rho^*)\right)]\\
&+4\gamma^{2}(H-1)^{2}\textstyle\sum_{t=T_r}^{v}\!\left(\mje{4mM_\sigma{}\!+\!8{M_f}{}}\!+\!{\nu^2}/{m}\right).
\end{align*}
Using $\gamma\le\frac{1}{5L_{G_\rho}(H-1)}$, hence $1-4\gamma^{2}(H-1)^{2}L_{G_\rho}^{2}\ge \frac{4}{5}$, which yields the desired result.
\end{proof}

\begin{lemma}\em\label{lemma: [Optimality gap single recursion]}
Consider Algorithm~\ref{algorithm: main algorithm-block-wise} and assume that Assumptions~\ref{assumption: main assumption 1},~\ref{assumption: main assumption 3}, and~\ref{assumption: main assumption 2} hold for some $m \ge 2$. If we choose a stepsize $\gamma > 0$ such that $\gamma \le \frac{1}{6L_{G_\rho}}$, then we have
\begin{align*}
&\mathbb{E}[\|\bar{\mathbf{x}}_{k+1}-\mathbf{x}_\rho^*\|^2]
\le\tfrac{3}{2}\gamma L_{G_\rho}\mathbb{E}[\tfrac{1}{m}\textstyle\sum_{i=1}^m\delta_{i,k}]+ \mathbb{E}[\|\bar{\mathbf{x}}_k-\mathbf{x}_\rho^*\|^2]-\gamma\mathbb{E}[G_{\rho}(\bar{\mathbf{x}}_{k})-G_{\rho}(\mathbf{x}_\rho^{*})]+\tfrac{3\gamma^2\nu^2}{m}.
\end{align*}
\end{lemma}
\begin{proof}
By the compact update rule~\eqref{equation: compact update rule}, the averaged iterate always satisfies $\bar{\mathbf{x}}_{k+1}=\bar{\mathbf{x}}_k-\gamma\frac{1}{m}\sum_{i=1}^m\nabla g_{i,\rho}(\mathbf{x}_{i,k},\xi_{i,k})$. Therefore, we obtain
\begin{align*}
&\|\bar{\mathbf{x}}_{k+1}-\mathbf{x}_\rho^*\|^2=\gamma^2\|\tfrac{1}{m}\textstyle\sum_{i=1}^m\nabla g_{i,\rho}(\mathbf{x}_{i,k},\xi_{i,k})\|^2+ \|\bar{\mathbf{x}}_k-\mathbf{x}_\rho^*\|^2-\tfrac{2\gamma}{m}\textstyle\sum_{i=1}^m\left\langle\bar{\mathbf{x}}_k-\mathbf{x}_\rho^*,\nabla g_{i,\rho}(\mathbf{x}_{i,k},\xi_{i,k})\right\rangle.
\end{align*}
Taking conditional expectations on both sides, using the fact that $\xi_{i,0},\ldots,\xi_{i,k}$ are independent samples for all $i=1,\ldots,m$, that $\bar{\mathbf{x}}_k$ and $\mathbf{x}_\rho^*$ are $\mathcal{F}_k$-measurable, and invoking Lemma~\ref{Lemma: bound for inner product term}, we obtain
\begin{align*}
&\mathbb{E}  [\|\bar{\mathbf{x}}_{k+1} - \mathbf{x}_\rho^*\|^2  | \mathcal{F}_k]  \le   \|\bar{\mathbf{x}}_k - \mathbf{x}_\rho^*\|^2   -  2\gamma\left(G_{\rho}(\bar{\mathbf{x}}_{k}) - G_{\rho}(\mathbf{x}_\rho^{*})\right)\\
& +\gamma^2\mathbb{E}  [\| \tfrac{1}{m}  \textstyle\sum_{i=1}^m  \nabla g_{i,\rho}(\mathbf{x}_{i,k},\xi_{i,k})\|^2  \big | \mathcal{F}_k]   +  \tfrac{\gamma L_{G_\rho}}{m}  \textstyle\sum_{i=1}^{m}  \delta_{i,k}.
\end{align*}
Taking another expectation on both sides and invoking Lemma~\ref{Lemma: the bound for nabla G_rho(mathbf{x}_k)}, we obtain
\begin{align*}
&\mathbb{E}[\|\bar{\mathbf{x}}_{k+1}-\mathbf{x}_\rho^*\|^2]\le \mathbb{E}[\|\bar{\mathbf{x}}_k-\mathbf{x}_\rho^*\|^2]+{\gamma L_{G_\rho}}\left(1+3{\gamma L_{G_\rho}}\right)\mathbb{E}[\tfrac{1}{m}\textstyle\sum_{i=1}^{m}\delta_{i,k}]\\
&\quad-2\gamma\left(1-{3\gamma L_{G_\rho}}\right)\mathbb{E}[G_{\rho}(\bar{\mathbf{x}}_{k})-G_{\rho}(\mathbf{x}_\rho^{*})]+{3\gamma^2\nu^2}/{m}.
\end{align*}
Based on the condition on $\gamma$, which is $\gamma \le \frac{1}{6L_{G_\rho}}$, we obtain that $-2\gamma\left(1 - 3\gamma L_{G_\rho}\right) \le -\gamma$ and $\gamma L_{G_\rho}\left(3\gamma L_{G_\rho} + 1\right) \le \frac{3}{2}\gamma L_{G_\rho}$. Invoking these bounds yields the desired result.
\end{proof}
In the following proposition, we establish a complexity bound for the penalized objective. This bound is used to prove the main results for problem~\eqref{problem: main problem1} in Proposition~\ref{theorem: main theorem for the non-iid} and Theorem~\ref{theorem: best rho}.
\begin{proposition}\em\label{proposition: bounds for the penalized function}
Let $m \ge 2$ and consider Algorithm~\ref{algorithm: main algorithm-block-wise}. Assume that Assumptions~\ref{assumption: main assumption 1},~\ref{assumption: main assumption 3}, and~\ref{assumption: main assumption 2} are satisfied and that $T_{r+1} - T_r= H$, for $r\ge0$ and some $H \ge 2$. Let \mje{$\rho\ge1$ and} $\mathbf{x}_\rho^* \in \arg\min_{\mathbf{x}} G_\rho(\mathbf{x})$. Choose a stepsize $\gamma > 0$ such that $\gamma \le \min\left\{\frac{1}{6L_{G_\rho}},\,\frac{1}{5L_{G_\rho}(H-1)}\right\}$. Then, the following holds.
\begin{align*}
\mathbb{E}\left[G_{\rho}(\widehat{\mathbf{x}}_K)-G_{\rho}(\mathbf{x}_\rho^{*})\right]\le&  {3\mathbb{E}[\|\bar{\mathbf{x}}_{0}-\mathbf{x}_\rho^*\|^2]}/{\fy{(K\gamma)}}+\!\tfrac{45}{2} L_{G_\rho}\!\gamma^{2}(H\!-\!1)^{2}\mje{(4mM_\sigma{}\!+\!8{M_f}{})}\\
&+\!\tfrac{9\nu^2}{m}\left(\tfrac{5}{2} L_{G_\rho}\gamma^{2}(H\!-\!1)^{2}\!+\!{\gamma }\right).
\end{align*}
\end{proposition}
\begin{proof}
As we mentioned, let $K \triangleq HR$. Summing both sides of Lemma~\ref{lemma: [Optimality gap single recursion]} over $k = 0, \ldots, K-1$, we obtain
\begin{align}
\textstyle\sum_{k=0}^{K-1}\mathbb{E}[\|\bar{\mathbf{x}}_{k+1}-\mathbf{x}_\rho^*\|^2]\le& \sum_{k=0}^{K-1}\mathbb{E}[\|\bar{\mathbf{x}}_{k}-\mathbf{x}_\rho^*\|^2]+\gamma\left( \tfrac{3}{2} L_{G_\rho}\textstyle\sum_{k=0}^{K-1}\mathbb{E}[\tfrac{1}{m}\sum_{i=1}^m\delta_{i,k}]\right.\notag\\
&\left.-\textstyle\sum_{k=0}^{K-1}\mathbb{E}[G_{\rho}(\bar{\mathbf{x}}_{k})-G_{\rho}(\mathbf{x}_\rho^{*})]\right)+ \tfrac{3K\gamma^2 \nu^2}{m}.\label{equation: the first one for the proposition}
\end{align}
Since $K=T_R$, applying Lemma~\ref{lemma: lemma for bounding the consensus} to the second term, we obtain
\begin{align*}
& \tfrac{3}{2} L_{G_\rho}\textstyle\sum_{t=1}^{R}\sum_{k=T_{t-1}}^{T_{t}-1}\mathbb{E}[\tfrac{1}{m}\sum_{i=1}^{m}\delta_{i,k}]-\textstyle\sum_{t=1}^{R}\sum_{k=T_{t-1}}^{T_{t}-1}\mathbb{E}[G_{\rho}(\bar{\mathbf{x}}_{k})-G_{\rho}(\mathbf{x}_\rho^{*})]\\&\me{\le\textstyle\sum_{t=1}^{R}\sum_{k=T_{t-1}}^{T_{t}-1}}\left(15\gamma^{2}(H-1)^{2}L_{G_\rho}^2-1\right)\mathbb{E}\left[\left(G_{\rho}(\bar{\mathbf{x}}_k)- G_{\rho}({\mathbf{x}}_\rho^*)\right)\right]\\
&+\!\textstyle\sum_{t=1}^{R}\!\sum_{k=T_{t-1}}^{T_{t}-1}\!\tfrac{15}{2} L_{G_\rho}\gamma^{2}(H-1)^{2}\!\left(\mje{4mM_\sigma{}\!+\!8{M_f}{}}\!+\!\tfrac{\nu^2}{m}\right).
\end{align*}
By our choice of $\gamma$, we have $\gamma\le \tfrac{1}{5L_{G_\rho}(H-1)}$, \Rme{which} implies that $15\gamma^{2}(H-1)^{2}L_{G_\rho}^2-1\le -\frac{1}{3}$. Substituting this bound into the preceding inequality, then into~\eqref{equation: the first one for the proposition}, and rearranging terms, we obtain
\begin{align*}
\tfrac{\gamma}{3}\textstyle\sum_{k=0}^{K-1}\mathbb{E}[G_{\rho}(\bar{\mathbf{x}}_{k})-G_{\rho}(\mathbf{x}_\rho^{*})]\le&\mathbb{E}[\|\bar{\mathbf{x}}_{0}-\mathbf{x}_\rho^*\|^2]+K\left(\tfrac{15}{2} L_{G_\rho}\gamma^{3}(H-1)^{2}\mje{(4mM_\sigma{}\!+\!8{M_f}{})}\right)\\
& +K(\tfrac{\nu^2}{m}(\tfrac{15}{2} L_{G_\rho}\gamma^{3}(H-1)^{2}+{3\gamma^2 })).
\end{align*}\mje{Dividing both sides by $\tfrac{K\gamma}{3}$ and using the convexity of $G_\rho$, Jensen's inequality gives}
$\mje{\mathbb{E}\!\left[G_\rho(\widehat{\mathbf{x}}_K)-G_\rho(\mathbf{x}_\rho^*)\right]}\mje{\leq\tfrac{1}{K}\textstyle\sum_{k=0}^{K-1}\mathbb{E}\!\left[G_\rho(\bar{\mathbf{x}}_k)-G_\rho(\mathbf{x}_\rho^*)\right],}$
\mje{which yields the desired result.}
\end{proof}

\me{The next proposition establishes upper and lower bounds on the global objective’s suboptimality, together with \mee{agent}-wise infeasibility bounds \fy{in addressing problem~\eqref{problem: main problem1}}. These bounds \fy{are} used to derive the communication complexity guarantees in Theorem~\ref{theorem: best rho}.}
\begin{proposition}\em\label{theorem: main theorem for the non-iid} Let Assumptions~\ref{assumption: main assumption 1},~\ref{assumption: main assumption 3}, and~\ref{assumption: main assumption 2} hold. Assume that \me{$\mathbf{x}^*\triangleq [x_1^*,\ldots,x_m^*]$} is the optimal solution of problem~\eqref{problem: main problem1} and \mje{$\rho\ge1$} is the penalty parameter. Suppose that $T_{r+1} - T_r= H$, for $r=0,\ldots,R-1$ and some $H \ge 2$. Choose a stepsize $\gamma > 0$ such that $\gamma \le \min\left\{\frac{1}{6L_{G_\rho}},\,\frac{1}{5L_{G_\rho}(H-1)}\right\}$. Then, the following hold

\noindent(i) \textbf{[Suboptimality-Upper Bound]} The expected optimality gap satisfies
\begin{align*}\mathbb{E}\left[ f(\me{\widehat{\mathbf{x}}_K})-f(\me{\mathbf{x}^*})\right]\!\le&\tfrac{9\gamma\nu^2}{m}+\!\tfrac{3\mathbb{E}\left[\left\|\bar{\mathbf{x}}_{0}-\mathbf{x}_\rho^*\right\|^2\right]}{K\gamma}+\!\tfrac{45}{2}\rho\gamma^{2}(H\!-\!1)^{2}\!\left(\mje{4mM_\sigma{}\!+\!8{M_f}{}}+\!\tfrac{\nu^2}{m}\right)\\
&\!+\!\!\me{\tfrac{45(L_f+\sigma(m-1))\gamma^{2}(H-1)^{2}}{2m}}\!\left(\mje{4mM_\sigma{}\!+\!8{M_f}{}}\!\!+\!\tfrac{\nu^2}{m}\!\right).
\end{align*}
\noindent(ii) \textbf{[Feasibility]} Moreover, for each $i \in \{1,\ldots,m\}$, the expected squared feasibility violation satisfies
\me{\begin{align*}
 \mathbb{E}\!\left[\left\| \widehat x_K^{(i)} \!- \!\Pi_{X_i}\!\left[ \widehat x_K^{(i)}\right]\right\|^2\right]\!\!\le &\tfrac{20m^2\|\nabla f(\mathbf{x}^*)\|^2}{\rho^2}\!+\!{{\tfrac{60m\mathbb{E}\left[\left\|\bar{\mathbf{x}}_{0}-\mathbf{x}_\rho^*\right\|^2\right]}{K\gamma\rho}}{}}+\!{\tfrac{{450}{} (L_f+\sigma(m-1))\gamma^{2}(H\!-\!1)^{2}\left(\mje{4mM_\sigma{}\!+\!8{M_f}{}}+\frac{\nu^2}{m}\right)}{\rho}}\!\\
&+\!{{{450m}{}\gamma^{2}(H\!-\!\!1)^{2}\!\left(\mje{4mM_\sigma{}\!+\!8{M_f}{}}+\!\!\tfrac{\nu^2}{m}\right)}{}}\!\!+\!{\tfrac{{180\gamma\nu^2}{}}{\rho}}.
\end{align*}}
\noindent(iii) \textbf{[Suboptimality-Lower Bound]} The expected lower bound for the optimality gap satisfies\me{\begin{align*}&-\textstyle\!\left\|\nabla f(\Rme{\mathbf{x}^*})\right\|\!\left(\tfrac{2m\|\nabla f(\mathbf{x}^*)\|}{\rho}+6\textstyle\sqrt{\tfrac{{\gamma\nu^2}{}}{\rho}}+2\sqrt{{\tfrac{3m\mathbb{E}\left[\left\|\bar{\mathbf{x}}_{0}-\mathbf{x}_\rho^*\right\|^2\right]}{K\gamma\rho}}{}}\right.\\
&\left.+2\sqrt{\tfrac{{45}{}(L_f+\sigma(m-1))\gamma^{2}\!(H-1)^{2}\!\left(\mje{4mM_\sigma{}\!+\!8{M_f}{}}+\frac{\nu^2}{m}\right)}{2\rho}}\right.\\
&\left.+2\sqrt{\tfrac{{45}{}m\gamma^{2}(H\!-\!1)^{2}\!\left(\mje{4mM_\sigma{}\!+\!8{M_f}{}}+\frac{\nu^2}{m}\right)}{2}}\right)\!\!\leq\!\mathbb{E}\left[ f(\me{\widehat{\mathbf{x}}_K})-f(\me{\mathbf{x}^*})\right].
\end{align*}}
\end{proposition}
\begin{proof}
\noindent(i) We know that $\mathbb{E}\left[G_\rho(\mathbf {x}_\rho^*)\right]\le\mathbb{E}\left[G_\rho(\mathbf {x}^*)\right]$, because $\mathbf {x}_\rho^*\triangleq \arg \min_{\mathbf x} G_\rho(\mathbf{x})$. Therefore, we obtain $\mathbb{E}\left[G_\rho(\widehat{\mathbf{x}}_K)- G_\rho(\mathbf{x}^*)\right] \le\mathbb{E}\left[G_\rho(\widehat{\mathbf{x}}_K)- G_\rho(\mathbf{x}_\rho^*)\right]$. Invoking the definition of $G_\rho(\mathbf{x})$, we obtain $\mathbb{E}\left[ f(\widehat{\mathbf{x}}_K)\right]+\rho H(\widehat{\mathbf{x}}_K)-\left(\mathbb{E}\left[ f(\mathbf{x}^*)\right]+\rho H(\mathbf{x}^*)\right) \le\mathbb{E}\left[G_\rho(\widehat{\mathbf{x}}_K)- G_\rho(\mathbf{x}_\rho^*)\right]$. Invoking the definition of $H(\mathbf{x})$, Proposition~\ref{proposition: bounds for the penalized function}, and letting $E$ denote the resulting bound on $\mathbb{E}[G_\rho(\widehat{\mathbf{x}}_K)-G_\rho(\mathbf{x}_\rho^*)]$, we obtain
\me{\begin{align}
&\mathbb{E}\left[ f(\widehat{\mathbf{x}}_K)\right]+ \tfrac{\rho}{2m}\textstyle\sum_{i=1}^m \mathbb{E}\left[{\bar{d}}_i^2\right]-\left(\mathbb{E}\left[ f(\mathbf{x}^*)\right]+ \tfrac{\rho}{2m}\textstyle\sum_{i=1}^m \left\|x^*_{i} - \Pi_{X_i}[x^*_{i}]\right\|^2\right)\le E.\notag\end{align}}We know that $x_i^*=\Pi_{X_i}[x_i^*]$, which implies that $\left\|x_i^*-\Pi_{X_i}[x_i^*]\right\|^2=0$, for all $i$. We obtain\me{\begin{align}
\mathbb{E}\!\left[ f(\widehat{\mathbf{x}}_K)\!\!-\!\!f(\mathbf{x}^*)\right]\!+ \tfrac{\rho}{2m}\!\textstyle\sum_{i=1}^m  \mathbb{E}\left[{\bar{d}}_i^2\right] \!\le E.\label{equation: the bound in the proof of the theorem for personalized one}
\end{align}}We know that \me{$ \frac{\rho}{2m}\sum_{i=1}^m {\bar{d}}_i^2 \ge 0$}, thus \me{$\mathbb{E}\!\left[ f(\widehat{\mathbf{x}}_K)-f(\mathbf{x}^*)\right]\! \le \!E$}. We should note that in the definition of $L_{G_\rho}$ we have $\rho$, therefore, invoking the definition of $L_{G_\rho}$ and substituting the expression of $E$, we obtain \me{the desired result.}

\noindent(ii) From {the} convexity of \me{$f$}, we may write $\nabla f(\mathbf{x}^*)^\top\mathbb{E}[\widehat{\mathbf{x}}_K-\mathbf{x}^*]\leq\mathbb{E}\!\left[ f(\widehat{\mathbf{x}}_K)-f(\mathbf{x}^*)\right]$. Note that \me{$\nabla f(\mathbf{x}^*)^\top\mathbb{E}[\widehat{\mathbf{x}}_K-\mathbf{x}^*]$} is not necessarily nonnegative because problem~\eqref{problem: main problem1} is a constrained problem. Adding and subtracting \me{$ \Pi_{X}[\widehat{\mathbf{x}}_K]$, where $X\triangleq \prod_{i=1}^m X_i$}, we obtain \me{$\nabla f(\mathbf{x}^*)^\top\mathbb{E}\left[\widehat{\mathbf{x}}_K-\mathbf{x}^*+ \Pi_{X}[\widehat{\mathbf{x}}_K]- \Pi_{X}[\widehat{\mathbf{x}}_K]\right]\le\mathbb{E}\!\left[ f(\widehat{\mathbf{x}}_K)-f(\mathbf{x}^*)\right]$.} In view of \me{$\Pi_{X}[\widehat{\mathbf{x}}_K] \in X$}, we have \me{$\nabla f(\mathbf{x}^*)^\top\mathbb{E}\left[ \Pi_{X}[\widehat{\mathbf{x}}_K]-\mathbf{x}^*\right]\geq 0$}. Thus
\me{$\nabla f(\mathbf{x}^*)^\top\!\mathbb{E}\left[\widehat{\mathbf{x}}_K\!\!-\!\Pi_{X}[\widehat{\mathbf{x}}_K\!]\right]\!\leq\!\mathbb{E}\!\left[ f(\widehat{\mathbf{x}}_K\!)\!-\!f(\mathbf{x}^*)\right].$}
%
%
By the Cauchy-Schwarz inequality and Jensen's inequality, we obtain $-\|\nabla f(\mathbf{x}^*)\|\mathbb{E}\left[{\bar d} \ \right]\!\le\!\mathbb{E}\!\left[ f(\widehat{\mathbf{x}}_K\!)\!-\!f(\mathbf{x}^*)\right]$. Substituting the preceding lower bound into~\eqref{equation: the bound in the proof of the theorem for personalized one}, we obtain $-\|\nabla f(\mathbf{x}^*)\|\mathbb{E}\left[\bar d\ \right]+\tfrac{\rho }{2m}\textstyle\sum_{i=1}^m \mathbb{E}\left[\bar d_i^2\right] \le E$. Since $\bar d^2=\textstyle\sum_{i=1}^m\bar d_i^2$, Jensen's inequality yields $\mathbb{E}[\bar d]\leq\sqrt{\mathbb{E}[\bar d^2]}=\sqrt{\textstyle\sum_{i=1}^m\mathbb{E}[\bar d_i^2]}$. Therefore, \me{$
 \tfrac{\rho }{2m}\textstyle\sum_{i=1}^m \mathbb{E}\left[\bar d_i^2\right]-\|\nabla f(\mathbf{x}^*)\|\sqrt{\textstyle\sum_{i=1}^m \mathbb{E}\left[\bar d_i^2\right]}\le E$.} \fy{Viewing the left-hand side of the} preceding inequality \fy{as a} quadratic \fy{function} in $a$, where \me{$a \triangleq \sqrt{\textstyle\sum_{i=1}^m \mathbb{E}\!\left[\bar d_i^2\right]}$}, \fy{we may write}
\me{$\sqrt{\textstyle\sum_{i=1}^m \mathbb{E}\!\left[\bar d_i^2\right]}\!\le\!\! \tfrac{2m\|\nabla f(\mathbf{x}^*)\|}{\rho}\!+\!2\sqrt{\tfrac{mE}{\rho}}$.}
Invoking the definition of $E$, we obtain
\begin{align}
\sqrt{\textstyle\sum_{i=1}^m\! \mathbb{E}\!\left[\bar d_i^2\right]}\!
\le& \tfrac{2m\|\nabla f(\mathbf{x}^*)\|}{\rho}\!+\!6\sqrt{\tfrac{{\gamma\nu^2}{}}{\rho}}\!+\!2\sqrt{{\tfrac{3m\mathbb{E}\left[\left\|\bar{\mathbf{x}}_{0}-\mathbf{x}_\rho^*\right\|^2\right]}{K\gamma\rho}}{}}+\!2\sqrt{\tfrac{{45}{}(L_f+\sigma(m-1))\gamma^{2}(H-1)^{2}\!\left(\mje{4mM_\sigma{}+8{M_f}{}}+\frac{\nu^2}{m}\right)}{2\rho}}\notag\\
&\!+\!2\sqrt{\tfrac{{45}{}m\gamma^{2}(H-1)^{2}\left(\mje{4mM_\sigma{}+8{M_f}{}}+\frac{\nu^2}{m}\right)}{2}}.\label{equation: the bound used in the lower bound}
\end{align}
\me{Since $\sqrt{ \mathbb{E}\!\left[\bar d_i^2\right]}\le\sqrt{\sum_{i=1}^m \mathbb{E}\!\left[\bar d_i^2\right]}$, for all $i$, \fy{we obtain}
\begin{align*}
 \mathbb{E}\!\left[\bar d_i^2\right]\!\le& \tfrac{20m^2\|\nabla f(\mathbf{x}^*)\|^2}{\rho^2}\!+\!{{\tfrac{60m\mathbb{E}\left[\left\|\bar{\mathbf{x}}_{0}-\mathbf{x}_\rho^*\right\|^2\right]}{K\gamma\rho}}{}}\!+\!{\tfrac{{180\gamma\nu^2}{}}{\rho}}+\!{\tfrac{{450}{} (L_f+\sigma(m-1))\gamma^{2}(H\!-\!1)^{2}\left(\mje{4mM_\sigma{}+8{M_f}{}}+\frac{\nu^2}{m}\right)}{\rho}}\!\\
&+\!{{{450m}{}\gamma^{2}(H\!-\!1)^{2}\left(\mje{4mM_\sigma{}\!+\!8{M_f}{}}+\tfrac{\nu^2}{m}\right)}{}}.
\end{align*}}
\noindent(iii) Recall that $-\|\nabla f(\mathbf{x}^*)\|\mathbb{E}[\bar d]\leq\mathbb{E}\left[f(\widehat{\mathbf{x}}_K)-f(\mathbf{x}^*)\right]$. Moreover, Jensen's inequality gives
$\mathbb{E}[\bar d]\leq\sqrt{\textstyle\sum_{i=1}^m\mathbb{E}[\bar d_i^2]}$. Therefore, $-{\|\nabla f(\mathbf{x}^*)\|}{}\sqrt{\textstyle\sum_{i=1}^m\mathbb{E}\left[\bar d_i^2\right]}\le\mathbb{E}\!\left[ f(\widehat{\mathbf{x}}_K\!)\!-\!f(\mathbf{x}^*)\right]$. Invoking the bound in~\eqref{equation: the bound used in the lower bound}, we obtain the desired result.
\end{proof}


{\bf{Proof sketch of Theorem~\ref{theorem: best rho}.}}
\noindent For the suboptimality upper-bound, we invoke part \noindent(i) of Proposition~\ref{theorem: main theorem for the non-iid}, setting the upper bound on
$\mathbb{E}\!\left[ f(\widehat{\mathbf{x}}_K) - f(\mathbf{x}^*) \right]$ to be at most $\epsilon$, then substituting $K=RH$, $\rho=\sqrt{R}$, and $\gamma=\gamma_0/\sqrt{R}$ yields part~(i). Applying the same substitutions to parts~(ii) and~(iii) of Proposition~\ref{theorem: main theorem for the non-iid}, together with $D\triangleq\|\nabla f(\mathbf{x}^*)\|^2$, yields the remaining results.

\section{Numerical Experiments}

We evaluate \mjee{PCE-FedAvg} on MNIST and CIFAR-10 using $m=10$ agents and a \mje{$C=10$} class softmax-regression model. Let $\mathbf{W}=[\mathbf{w}_1,\ldots,\mathbf{w}_\mje{C}]\in\mathbb{R}^{D\times \mje{C}}$ denote the model parameters. Agent $i$ possesses a local dataset $\mathcal{D}_i=\{(\phi_{i,n},y_{i,n})\}_{n=1}^{N_i}$ and employs the softmax cross-entropy loss
\begin{align*}
f_i(\mathbf{W})\triangleq-\tfrac{1}{N_i}\textstyle\sum_{n=1}^{N_i}\log\biggl(\tfrac{\exp(\mathbf{w}_{y_{i,n}}^\top\phi_{i,n})}{\sum_{\ell=1}^{\mje{C}}\exp(\mathbf{w}_{\ell}^\top\phi_{i,n})}\biggr).
\end{align*}
We generate heterogeneous local datasets using a class-wise Dirichlet distribution. The same class proportions are used to partition the training and test data among the agents. Each agent maintains a personalized model $\mathbf{W}_i$, and we define $\bar{\mathbf{W}}\triangleq\tfrac{1}{m}\textstyle\sum_{i=1}^m\mathbf{W}_i$. To impose agent-specific sparsity requirements, we consider $X_i^{\mathrm{sp}}\triangleq\bigl\{\mathbf{W}\in\mathbb{R}^{D\times\mje{C}}:\|\mathbf{W}\|_1\leq\tau_i\bigr\}$. The resulting constrained personalized FL problem is
\begin{align*}
\min_{\mathbf{W}_1,\ldots,\mathbf{W}_m}\quad
&\tfrac{1}{m}\textstyle\sum_{i=1}^m\left(f_i(\bar{\mathbf{W}})+\tfrac{\sigma_i}{2}\|\mathbf{W}_i-\bar{\mathbf{W}}\|_F^2\right)\\
\text{s.t.}\quad &\mathbf{W}_i\in X_i^{\mathrm{sp}},\qquad i\in[m].
\end{align*}

We compare \mjee{PCE-FedAvg} with penalized variants of FedAvg~\cite{mcmahan2017communication}, FedProx~\cite{li2020federated}, and SCAFFOLD~\cite{karimireddy2020scaffold}. To provide a fair comparison, we incorporate the corresponding local distance penalty into each baseline rather than applying the standard unconstrained methods. Specifically, the baselines apply their respective update rules to
\begin{align*}
\min_{\mathbf{W}\in\mathbb{R}^{D\times \mje{C}}}\quad
\tfrac{1}{m}\textstyle\sum_{i=1}^m\left(f_i(\mathbf{W})+\tfrac{\rho}{2}\operatorname{dist}^2(\mathbf{W},X_i^{\mathrm{sp}})\right),
\end{align*}
so that agent $i$ includes the feasibility-correction direction $\rho\bigl(\mathbf{W}-\Pi_{X_i^{\mathrm{sp}}}[\mathbf{W}]\bigr)$ in its local update. By incorporating the same type of penalization into all methods, this comparison demonstrates that the performance of \mjee{PCE-FedAvg} is not due merely to penalizing the local constraints. Rather, it highlights the central features of the proposed method: the multi-block formulation, the cross-estimate mechanism, and the relaxation of consensus among the personalized models. All methods use the same fixed penalty parameter $\rho=\sqrt{R}$, full participation, and minibatch size $128$. The stepsizes of all methods are tuned separately and then held fixed throughout the communication rounds. The FedProx parameters are $\mu=7$ for MNIST and $\mu=20$ for CIFAR-10.  All methods are initialized from the same point lying strictly outside every agent's constraint set. We measure the squared infeasibility of agent $i$ at round $r$ by
\begin{align*}
\operatorname{dist}^2(\mathbf{W}_{i,r},X_i^{\mathrm{sp}})
\triangleq\bigl\|\mathbf{W}_{i,r}-\Pi_{X_i^{\mathrm{sp}}}[\mathbf{W}_{i,r}]\bigr\|^2.
\end{align*}

\mje{The assumptions used in our analysis are satisfied in this experiment. In particular, the softmax cross-entropy loss is convex and smooth because the input features are bounded. Each set $X_i^{\mathrm{sp}}$ is nonempty, closed, and convex, and its Euclidean projection is well defined. Moreover, uniform minibatch sampling produces unbiased stochastic gradients with bounded variance. Therefore, Theorem~\ref{theorem: best rho} predicts $\mathcal{O}(R^{-1/2})$ convergence of the expected objective gap and $\mathcal{O}(R^{-1})$ convergence of the agent-wise squared infeasibility.}

Figure~\ref{fig:infeasibility-comparison} reports the agent-wise squared infeasibility for Agents 1, 5, and 10. Although all methods start from the same infeasible point and use the same penalty parameter, \mjee{PCE-FedAvg} drives the infeasibility of each displayed agent toward zero, whereas the penalized baselines either stabilize at nonzero infeasibility levels or exhibit increasing constraint violations over the observed communication rounds. This comparison shows that the improvement is not due merely to adding a distance penalty. FedAvg, FedProx, and SCAFFOLD still learn a single shared model. At the beginning of every communication round, the server replaces the local models with their aggregate, so the shared model must balance the learning objective against all agent-specific penalties simultaneously. Consequently, it may accept violations for some agents, particularly those with more restrictive constraint sets. In contrast, \mjee{PCE-FedAvg} retains a separate block for every personalized model. The server averages corresponding copies of each block but does not collapse the different blocks into one model. Therefore, the personalized blocks can remain distinct across communication rounds, and the feasibility correction associated with $X_i^{\mathrm{sp}}$ acts directly on agent $i$'s block without requiring other agents to know $X_i^{\mathrm{sp}}$. This block-wise mechanism enables \mjee{PCE-FedAvg} to satisfy heterogeneous local constraints while preserving collaboration through the averaged model $\bar{\mathbf{W}}$.

Figure~\ref{fig:global-objective-comparison} shows that \mjee{PCE-FedAvg} achieves global objective values comparable to those of the penalized baselines while attaining substantially smaller agent-wise infeasibility. Thus, the feasibility improvement provided by the multi-block structure does not result in a substantial deterioration of the global objective. Figure~\ref{fig:local-step-comparison} illustrates the effect of the number of local iterations on \mjee{PCE-FedAvg}. Increasing the number of local iterations generally accelerates the reduction of the global objective for both datasets.

\mje{Finally, Figure~\ref{fig:heterogeneity-infeasibility} illustrates the effect of increasing statistical heterogeneity on Agent 1's squared infeasibility. \mjee{PCE-FedAvg} drives the infeasibility toward zero at all three heterogeneity levels. In contrast, although the penalized shared-model baselines converge to zero in the nearly IID case, they converge to nonzero violations under medium heterogeneity, while their violations increase over the displayed communication rounds under very high heterogeneity. This behavior highlights the advantage of retaining separate personalized blocks when agents have increasingly different data distributions and local constraints.}

\begin{figure}[H]
\centering
\setlength{\tabcolsep}{2pt}
\renewcommand{\arraystretch}{1.4}
\begin{tabular}{c|ccc}
& \scriptsize Agent 1 & \scriptsize Agent 5 & \scriptsize Agent 10\\
\hline
\raisebox{-.5\height}{\rotatebox[origin=c]{90}{\scriptsize CIFAR-10}}
&\raisebox{-.6\height}{\includegraphics[width=.24\columnwidth]{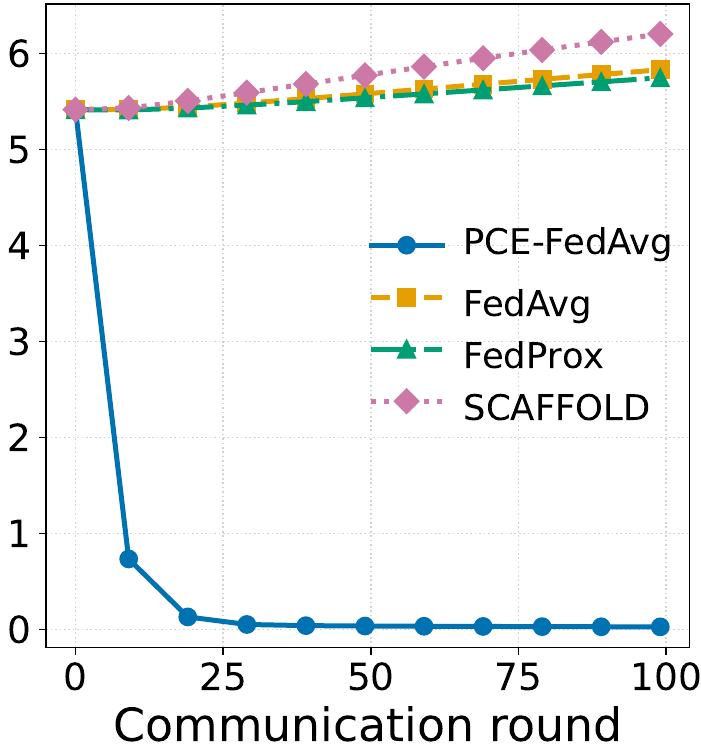}}
&\raisebox{-.6\height}{\includegraphics[width=.24\columnwidth]{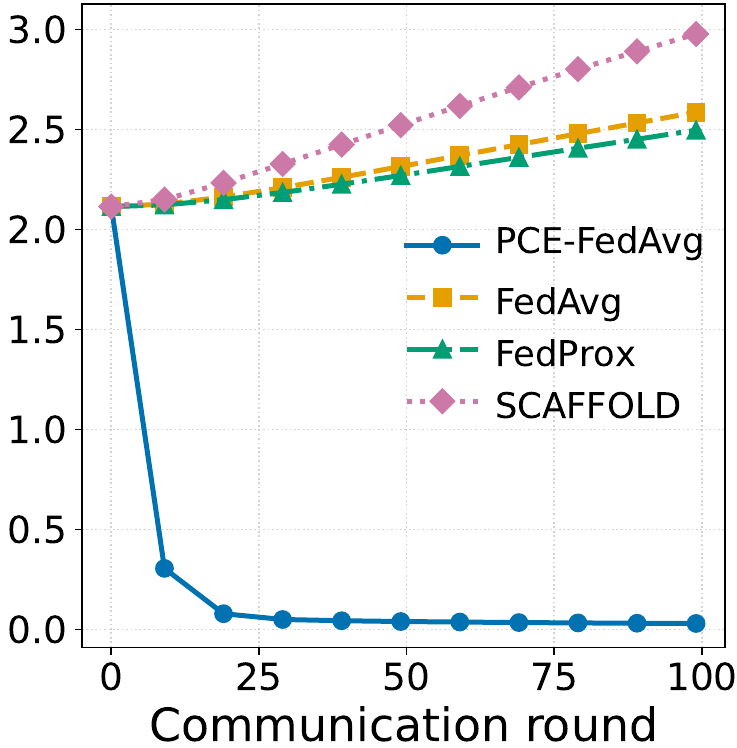}}
&\raisebox{-.6\height}{\includegraphics[width=.24\columnwidth]{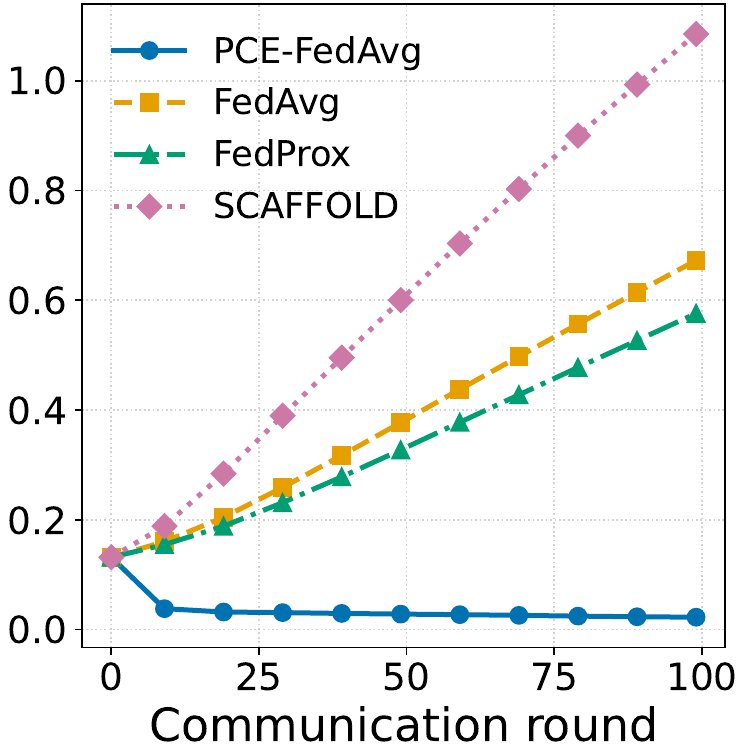}}\\[3mm]
\hline
\raisebox{-.5\height}{\rotatebox[origin=c]{90}{\scriptsize MNIST}}
&\raisebox{-.6\height}{\includegraphics[width=.24\columnwidth]{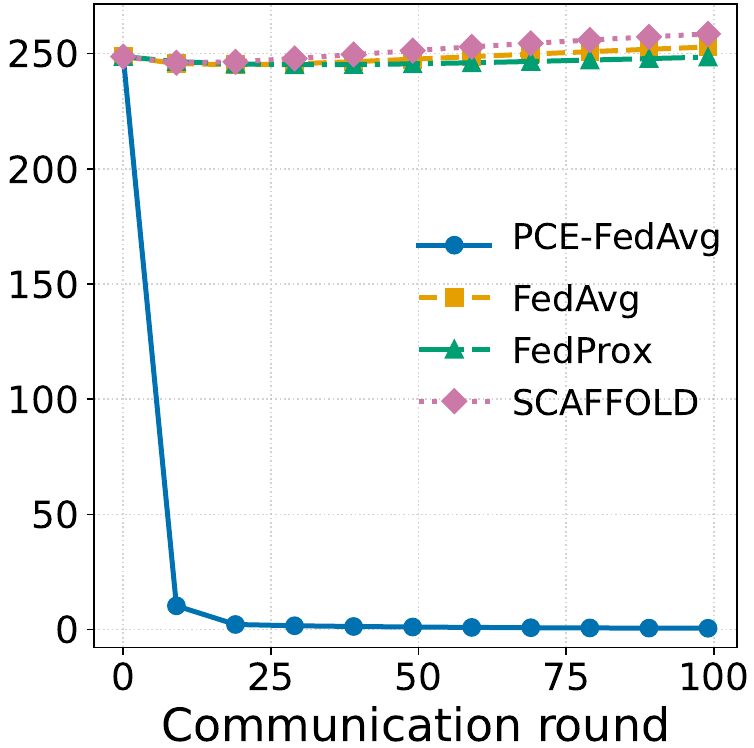}}
&\raisebox{-.6\height}{\includegraphics[width=.24\columnwidth]{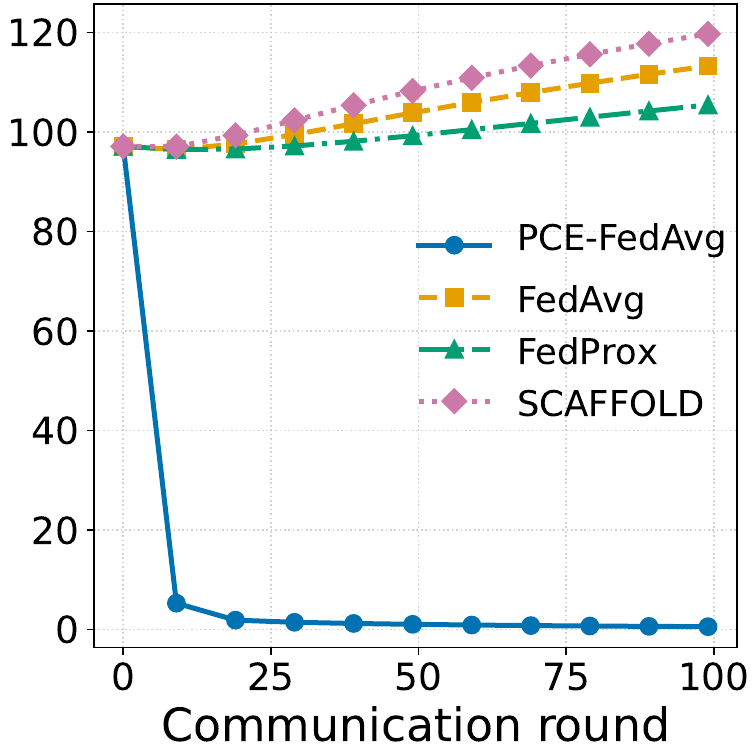}}
&\raisebox{-.6\height}{\includegraphics[width=.24\columnwidth]{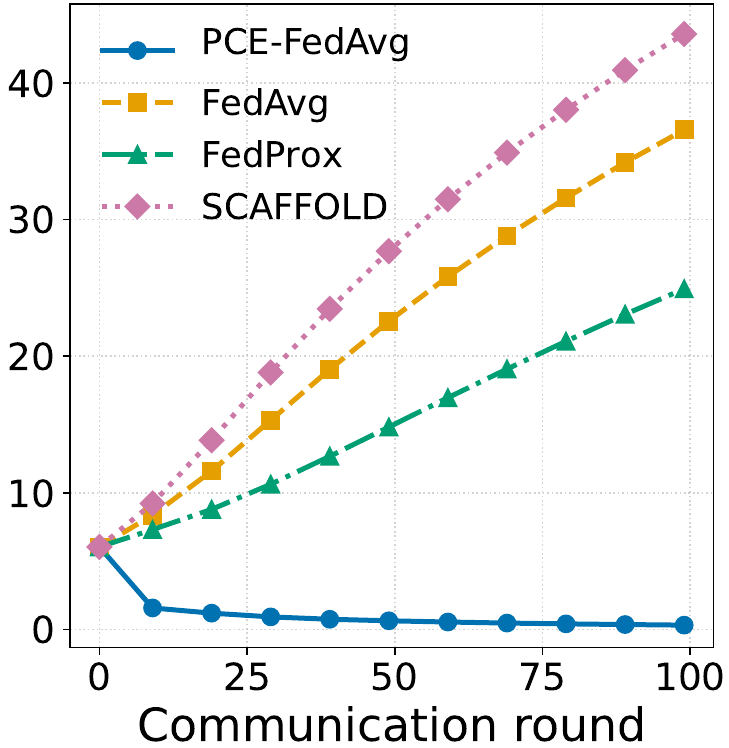}}\\[3mm]
\end{tabular}
\caption{\mje{Agent-wise squared infeasibility on CIFAR-10 and MNIST. The ordinate is $\operatorname{dist}^2(\mathbf{W}_{i,r},X_i^{\mathrm{sp}})$, where zero indicates satisfaction of Agent $i$'s local sparsity constraint. The columns show Agents 1, 5, and 10, and the rows correspond to CIFAR-10 and MNIST.}}
\label{fig:infeasibility-comparison}
\end{figure}

\begin{figure}[H]
\vspace{-0.16 in}
\centering
\setlength{\tabcolsep}{2pt}
\renewcommand{\arraystretch}{1.4}
\begin{tabular}{cc}
\scriptsize CIFAR-10 & \scriptsize MNIST\\
\hline
\raisebox{-.6\height}{\includegraphics[width=.46\columnwidth]{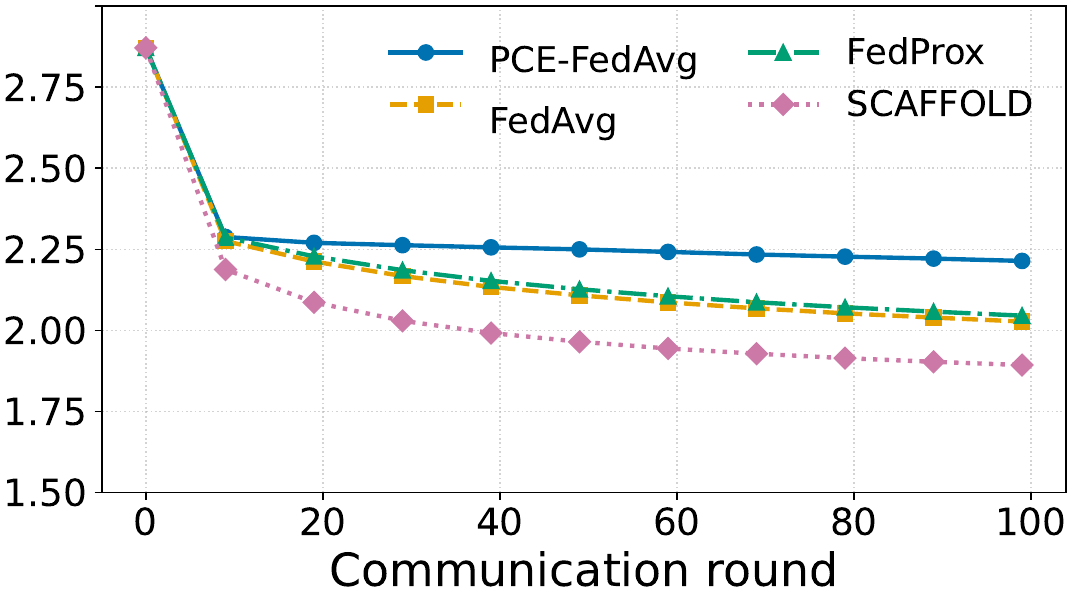}}
&\raisebox{-.6\height}{\includegraphics[width=.46\columnwidth]{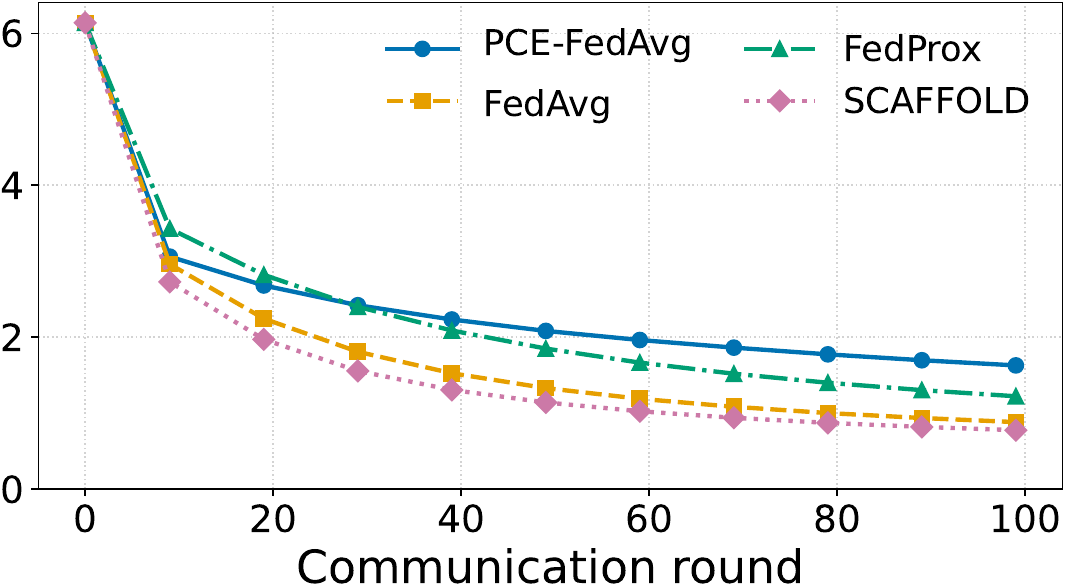}}\\
\end{tabular}
\caption{\mje{Comparison of the global objective values of all methods on CIFAR-10 and MNIST.}}
\label{fig:global-objective-comparison}
\end{figure}

\begin{figure}[H]
\centering
\setlength{\tabcolsep}{2pt}
\renewcommand{\arraystretch}{1.4}
\begin{tabular}{cc}
\scriptsize CIFAR-10 & \scriptsize MNIST\\
\hline
\raisebox{-.6\height}{\includegraphics[width=.36\columnwidth]{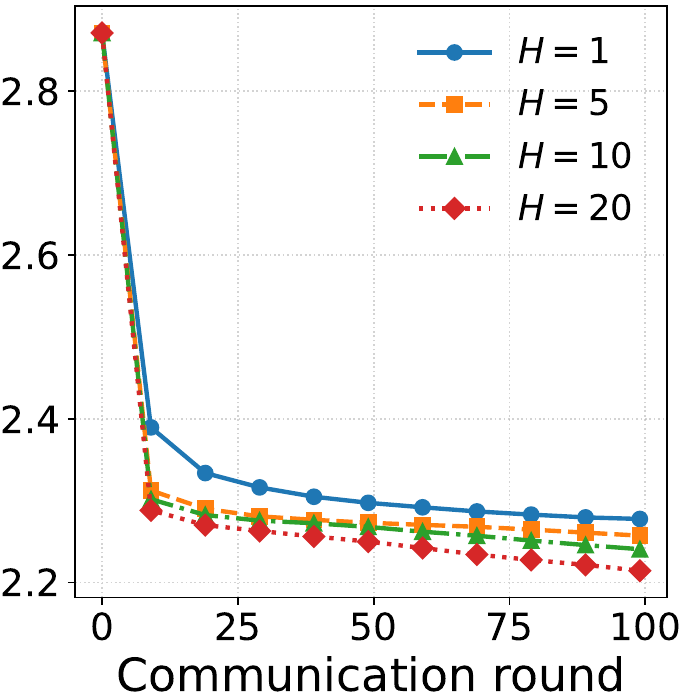}}
&\raisebox{-.6\height}{\includegraphics[width=.34\columnwidth]{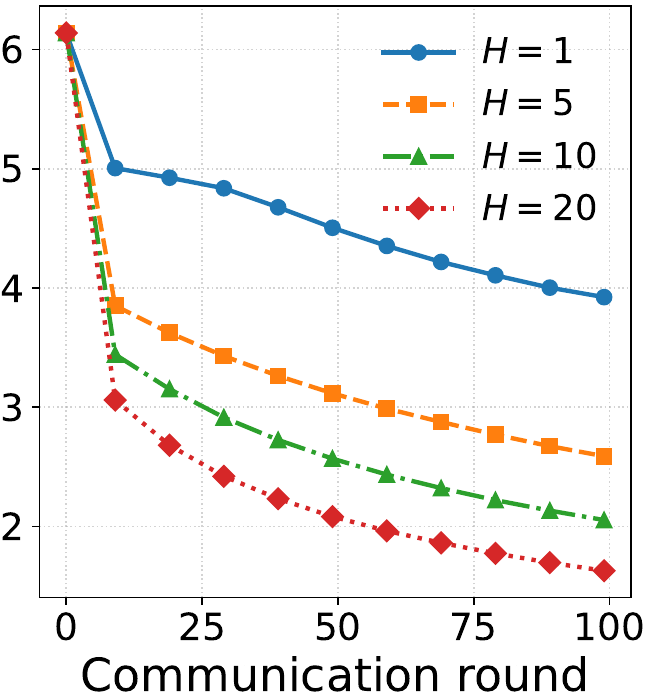}}\\
\end{tabular}
\caption{Global objective values of \mjee{PCE-FedAvg} for different numbers of local iterations.}
\label{fig:local-step-comparison}
\end{figure}

\begin{figure}[t]
\vspace{-0.2in}
\centering
\setlength{\tabcolsep}{2pt}
\renewcommand{\arraystretch}{1.4}
\begin{tabular}{ccc}
\mje{\scriptsize Nearly IID ($\alpha=100$)}
& \mje{\scriptsize Medium ($\alpha=1$)}
& \mje{\scriptsize Very high ($\alpha=0.005$)}\\
\hline
\raisebox{-.6\height}{\includegraphics[width=.308\columnwidth]{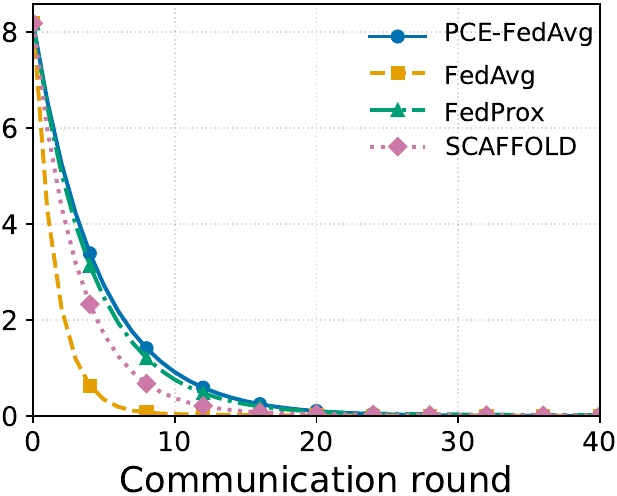}}
&\raisebox{-.6\height}{\includegraphics[width=.308\columnwidth]{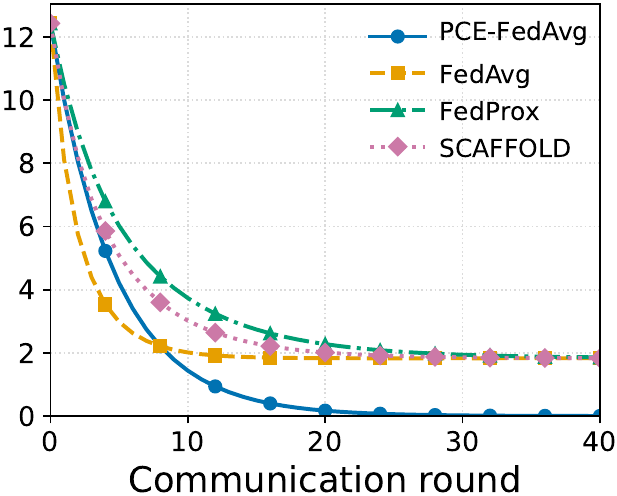}}
&\raisebox{-.6\height}{\includegraphics[width=.308\columnwidth]{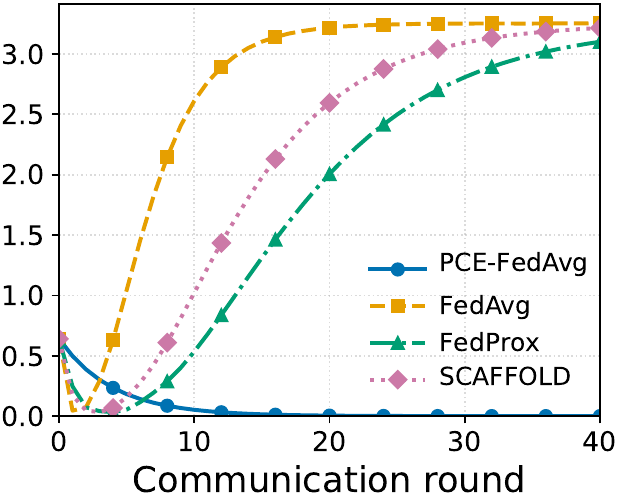}}\\[3mm]
\end{tabular}
\caption{\mje{Squared infeasibility of Agent 1 on MNIST under increasing statistical heterogeneity. From left to right, the Dirichlet parameter decreases from $\alpha=100$ to $\alpha=0.005$, corresponding to increasingly heterogeneous local data. All methods use the same initialization and penalty parameter.}}
\label{fig:heterogeneity-infeasibility}
\end{figure}

\section{CONCLUSION}

We studied a personalized constrained FL problem in which agent-specific models are coupled through a collaborative objective without requiring consensus or a common feasible model. We proposed \mjee{PCE-FedAvg}, a penalty-based method that uses multi-block cross-estimates and block-wise server aggregation while keeping each agent's constraint information private. We established finite-time upper and lower bounds on the objective gap and agent-wise squared infeasibility, yielding communication complexities of $\mathcal{O}(\epsilon^{-2})$ and $\mathcal{O}(\epsilon^{-1})$, respectively. Experiments on MNIST and CIFAR-10 with heterogeneous data and agent-specific $\ell_1$ constraints show that \mjee{PCE-FedAvg} substantially reduces agent-wise infeasibility while attaining global objective values comparable to penalized variants of FedAvg, FedProx, and SCAFFOLD. \mjee{Directions for future research include weakening the requirement of globally known stepsizes and penalty parameters by allowing local parameter selection, as well as extending the convergence analysis to nonconvex settings.}





%
%


\bibliographystyle{IEEEtran}
\bibliography{L-CSS-Bib}

@article{li2020federated,
  title={Federated optimization in heterogeneous networks},
  author={Li, Tian and Sahu, Anit Kumar and Zaheer, Manzil and Sanjabi, Maziar and Talwalkar, Ameet and Smith, Virginia},
  journal={Proceedings of Machine Learning and Systems},
  volume={2},
  pages={429--450},
  year={2020}
}

@article{predd2006distributed,
  title={Distributed Learning in Wireless Sensor Networks},
  author={Predd, Joel B. and Kulkarni, Sanjeev R. and Poor, H. Vincent},
  journal={IEEE Signal Processing Magazine},
  volume={23},
  number={4},
  pages={56--69},
  year={2006}
}

@article{lin2022ppfl_tsg,
  title={Privacy-Preserving Household Characteristic Identification With Federated Learning Method},
  author={Lin, Jun and Ma, Jin and Zhu, Jianguo},
  journal={IEEE Transactions on Smart Grid},
  volume={13},
  number={2},
  pages={1088--1099},
  year={2022}
}

@article{nedic2009distributed,
  title={Distributed Subgradient Methods for Multi-Agent Optimization},
  author={Nedi{\'c}, Angelia and Ozdaglar, Asuman},
  journal={IEEE Transactions on Automatic Control},
  volume={54},
  number={1},
  pages={48--61},
  year={2009}
}

@article{nedic2010constrained,
  title={Constrained Consensus and Optimization in Multi-Agent Networks},
  author={Nedi{\'c}, Angelia and Ozdaglar, Asuman and Parrilo, Pablo A.},
  journal={IEEE Transactions on Automatic Control},
  volume={55},
  number={4},
  pages={922--938},
  year={2010}
}

@inproceedings{yuan2021federated,
  title={Federated composite optimization},
  author={Yuan, Honglin and Zaheer, Manzil and Reddi, Sashank},
  booktitle={International Conference on Machine Learning},
  volume={139},
  pages={12253--12266},
  year={2021},
  organization={PMLR}
}

@article{nguyen2023geometric,
  title={Geometric convergence of distributed heavy-ball {Nash} equilibrium algorithm over time-varying digraphs with unconstrained actions},
  author={Nguyen, Duong Thuy Anh and Nguyen, Duong Tung and Nedi{\'c}, Angelia},
  journal={IEEE Control Systems Letters},
  volume={7},
  pages={1963--1968},
  year={2023},
  publisher={IEEE}
}

@inproceedings{mcmahan2017communication,
  title={Communication-efficient learning of deep networks from decentralized data},
  author={McMahan, Brendan and Moore, Eider and Ramage, Daniel and Hampson, Seth and y Arcas, Blaise Aguera},
  booktitle={International Conference on Artificial Intelligence and Statistics},
  volume= {54},
  pages={1273--1282},
  year={2017},
  organization={PMLR}
}

@inproceedings{majcherczyk2021flow,
  title={{Flow-FL}: Data-driven federated learning for spatio-temporal predictions in multi-robot systems},
  author={Majcherczyk, Nathalie and Srishankar, Nishan and Pinciroli, Carlo},
  booktitle={IEEE International Conference on Robotics and Automation (ICRA)},
  pages={8836--8842},
  year={2021},
  organization={IEEE}
}

@article{berkenkamp2017safe,
  title={Safe model-based reinforcement learning with stability guarantees},
  author={Berkenkamp, Felix and Turchetta, Matteo and Schoellig, Angela and Krause, Andreas},
  journal={Advances in Neural Information Processing Systems},
  volume={30},
  pages= {908--918},
  year={2017}
}

@inproceedings{karimireddy2020scaffold,
  title={{SCAFFOLD}: Stochastic controlled averaging for federated learning},
  author={Karimireddy, Sai Praneeth and Kale, Satyen and Mohri, Mehryar and Reddi, Sashank and Stich, Sebastian and Suresh, Ananda Theertha},
  booktitle={International Conference on Machine Learning},
    volume={119},
  pages={5132--5143},
  year={2020},
  organization={PMLR}
}

@article{kairouz2021advances,
  title={Advances and open problems in federated learning},
  author={Kairouz, Peter and McMahan, H Brendan},
  journal={Foundations and Trends in Machine Learning},
  volume={14},
  number={1-2},
  pages={1--210},
  year={2021},
  publisher={Emerald Publishing Limited}
}

@article{he2024federated,
  title={Federated learning with convex global and local constraints},
  author={He, Chuan and Peng, Le and Sun, Ju},
  journal={Transactions on machine learning research},
  year={2024}
}

@inproceedings{akgun2024projected,
  title={Projected push-pull for distributed constrained optimization over time-varying directed graphs},
  author={Akg{\"u}n, Orhan Eren and Day{\i}, Arif Kerem and Gil, Stephanie and Nedi{\'c}, Angelia},
  booktitle={2024 American Control Conference (ACC)},
  pages={2082--2089},
  year={2024},
  organization={IEEE}
}

@article{fallah2020personalized,
  title={Personalized federated learning with theoretical guarantees: A model-agnostic meta-learning approach},
  author={Fallah, Alireza and Mokhtari, Aryan and Ozdaglar, Asuman},
  journal={Advances in Neural Information Processing Systems},
  volume={33},
  pages={3557--3568},
  year={2020}
}

@article{t2020personalized,
  title={Personalized federated learning with {Moreau} envelopes},
  author={T Dinh, Canh and Tran, Nguyen and Nguyen, Josh},
  journal={Advances in Neural Information Processing Systems},
  volume={33},
  pages={21394--21405},
  year={2020}
}

@inproceedings{li2021ditto,
  title={Ditto: Fair and robust federated learning through personalization},
  author={Li, Tian and Hu, Shengyuan and Beirami, Ahmad and Smith, Virginia},
  booktitle={International Conference on Machine Learning},
  volume={139},
  pages={6357--6368},
  year={2021},
  organization={PMLR}
}

@inproceedings{collins2021exploiting,
  title={Exploiting shared representations for personalized federated learning},
  author={Collins, Liam and Hassani, Hamed and Mokhtari, Aryan and Shakkottai, Sanjay},
  booktitle={International Conference on Machine Learning},
  volume={139},
  pages={2089--2099},
  year={2021},
  organization={PMLR}
}

@inproceedings{shamsian2021personalized,
  title={Personalized federated learning using hypernetworks},
  author={Shamsian, Aviv and Navon, Aviv and Fetaya, Ethan and Chechik, Gal},
  booktitle={International Conference on Machine Learning},
  volume={139},
  pages={9489--9502},
  year={2021},
  organization={PMLR}
}

@article{ebrahimi2026self,
  title={Self-tuned regularized federated methods with guarantees for optimal solution selection},
  author={Ebrahimi, Mohammadjavad and Qiu, Yuyang and Cui, Shisheng and Yousefian, Farzad},
  journal={IEEE Control Systems Letters},
  volume={10},
  pages={1909--1914},
  year={2026},
  publisher={IEEE}
}

\end{document}